\documentclass[11pt]{article}

\usepackage[a4paper, left=3cm, right=3cm, top=3cm, bottom=3cm]{geometry}%

\usepackage{preamble}
\usepackage{times}

\usepackage[utf8]{inputenc} % allow utf-8 input
\usepackage[T1]{fontenc}    % use 8-bit T1 fonts
\usepackage{hyperref}       % hyperlinks
\usepackage{url}            % simple URL typesetting
\usepackage{booktabs}       % professional-quality tables
\usepackage{amsfonts}       % blackboard math symbols
\usepackage{amstext}
\usepackage{nicefrac}       % compact symbols for 1/2, etc.
\usepackage{microtype}      % microtypography

\usepackage{footnote}
\usepackage{tabularx}
\usepackage{moresize}
\usepackage{booktabs}
\usepackage{pifont}
\usepackage{placeins}

\newcommand{\cmark}{\ding{51}}%
\newcommand{\xmark}{\ding{55}}%
\newcommand{\Df}{\mathcal{D}_f}
\newcommand{\Dg}{\mathcal{D}_g}

\newcommand{\Dfs}{\mathcal{D}_{f^*}}
\newcommand{\Dfse}{\mathcal{D}_{f^*_{|\varepsilon}}}
\newcommand{\Dfsm}{-\Dfs}
\newcommand{\Dfsme}{-\Dfse}

\title{Dual optimization for convex constrained objectives\\without the gradient-Lipschitz assumption}

\author{
  Martin Bompaire % \\
  \thanks{CMAP, UMR 7641, École Polytechnique CNRS, Paris, France}\\
  \and
  St\'ephane Ga\"{\i}ffas % \\
  \thanks{LPSM, UMR 8001, Universit\'e Paris Diderot, Paris, France}\\
  \and
  Emmanuel Bacry % \\
  \footnotemark[1]~
  \thanks{CNRS, CEREMADE Université Paris-Dauphine PSL, Paris France}
}

\date{}

\begin{document}

\maketitle

\begin{abstract}
  The minimization of convex objectives coming from linear supervised learning problems, such as
  penalized generalized linear models, can be formulated as finite sums of convex functions.
  For such problems, a large set of stochastic first-order solvers based on the idea of variance
  reduction are available and combine both computational efficiency and sound theoretical
  guarantees (linear convergence rates) \cite{johnson2013accelerating},
  \cite{schmidt2013minimizing}, \cite{shalev2013stochastic}, \cite{defazio2014saga}.
  Such rates are obtained under both gradient-Lipschitz and strong convexity
  assumptions.
  Motivated by learning problems that do not meet the gradient-Lipschitz assumption, such as
  linear Poisson regression, we work under another smoothness assumption, and
  obtain a linear convergence rate for a shifted version of Stochastic Dual Coordinate Ascent
  (SDCA) \cite{shalev2013stochastic} that improves the current state-of-the-art.
  Our motivation for considering a solver working on the Fenchel-dual problem comes from the
  fact that such objectives include many linear constraints, that are easier to deal with in the
  dual.
  Our approach and theoretical findings are validated on several datasets, for Poisson regression
  and another objective coming from the negative log-likelihood of the Hawkes process, which is a
  family of models which proves extremely useful for the modeling of information
  propagation in social networks and causality
  inference \cite{de2016learning}, \cite{farajtabar2015coevolve}.
\end{abstract}

\section{Introduction} % (fold)
\label{sec:introduction}

In the recent years, much effort has been made to minimize strongly convex finite sums with first
order information.
Recent developments, combining both numerical efficiency and sound theoretical guarantees, such as
linear convergence rates, include SVRG \cite{johnson2013accelerating},
SAG \cite{schmidt2013minimizing}, SDCA \cite{shalev2013stochastic} or
SAGA \cite{defazio2014saga} to solve the following problem:
\begin{equation}
  \label{eq:general_primal}
  \min_{w \in \R^d} \;\; \frac{1}{n} \sum_{i=1}^n \varphi_i(w) + \lambda g(w),
\end{equation}
where the functions $\varphi_i$ correspond to a loss computed at a sample $i$ of the dataset, and
$g$ is a (eventually non-smooth) penalization.
However, theoretical guarantees about these algorithms, such as linear rates guaranteeing a
numerical complexity $O(\log(1 / \varepsilon))$ to obtain a solution
$\varepsilon$-distant to the minimum, require both strong convexity of
$\frac{1}{n} \sum_{i=1}^n \varphi_i + \lambda g$ and a gradient-Lipschitz property on each
$\varphi_i$, namely $\| \varphi_i'(x) - \varphi_i'(y) \| \leq L_i \| x - y \|$ for any
$x, y \in \R^d$, where $\| \cdot \|$ stands for the Euclidean norm on $\R^d$ and $L_i > 0$ is the
Lipschitz constant.
However, some problems, such as the linear Poisson regression, which is of practical importance
in statistical image reconstruction among others (see \cite{bertero2009image} for more than a
hundred references) do not meet such a smoothness assumption.
Indeed, we have in this example $\varphi_i(w) = w^\top x_i - y_i \log(w^\top x_i)$ for
$i = 1, \ldots, n$ where $x_i \in \R^d$ are the features vectors and $y_i \in
\N$ are the labels, and where the model-weights must satisfy the linear
constraints $w^\top x_i > 0$ for all $i=1, \ldots, n$.

Motivated by machine learning problems described in Section~\ref{sec:applications} below,
that do not satisfy the gradient-Lipschitz assumption, we consider a more specific task relying
on a new smoothness assumption.
Given convex functions $f_i : \Df \rightarrow \R$ with $\Df = (0, +\infty)$ such that
$\lim_{t \rightarrow 0} f_i(t) = + \infty$,
a vector $\psi \in \R^d$, features vectors $x_1, \ldots, x_n \in \R^d$ corresponding to the rows
of a matrix $X$ we consider the objective
\begin{equation}
  \label{eq:primal}
  \min_{w \in \Pi(X)} P(w)
  \quad \text{where} \quad P(w) = \psi^\top w + \frac{1}{n} \sum_{i=1}^n
  f_i(w^\top x_i) + \lambda g(w),
\end{equation}
where $\lambda > 0$, $g : \R^d \rightarrow \R$ is a $1$-strongly convex function and
$\Pi(X)$ is the open polytope
\begin{equation}
  \label{eq:feasible_polytope}
  \Pi(X) = \{ w \in \R^d \; : \; \forall i \in \{1 , \dots, n\}, \; w^\top x_i > 0 \},
\end{equation}
that we assume to be non-empty.
Note that the linear term $\psi^\top w$ can be included in the regularization $g$ but
the problem stands clearer if it is kept out.
\begin{definition}
\label{hyp:new-gradient-assumption}
  We say that a function $f : \Df \subset \R \rightarrow \R$ is $L$-\emph{$\log$ smooth},
  where $L > 0$, if it is a differentiable and strictly monotone convex function that satisfies
  \begin{equation*}
    | f'(x) - f'(y) | \leq \frac{1}{L} f'(x) f'(y) | x - y |
  \end{equation*}
  for all $x, y \in \Df$.
\end{definition}
We detail this property and its specificities in Section~\ref{sec:log-smoothness}.
All along the chapter, we assume that the functions $f_i$ are $L_i$-$\log$ smooth.
Note also that the Poisson regression objective fits in this setting, where $f_i(x) = -y_i \log x$
is $y_i$-$\log$ smooth and $\psi = \frac 1n \sum_{i=1}^n x_i$.
See Section~\ref{sub:linear-poisson-regression} below for more details.

\paragraph{Related works.} % (fold)
\label{par:related_works}

% paragraph related_works (end)
Standard first-order batch solvers (non stochastic) for composite convex objectives are ISTA and
its accelerated version FISTA \cite{beck2009fast} and first-order stochastic solvers are mostly
built on the idea of Stochastic Gradient Descent (SGD) \cite{robbins1951stochastic}.
Recently, stochastic solvers based on a combination of SGD and the Monte-Carlo technique of
variance reduction \cite{schmidt2013minimizing}, \cite{shalev2013stochastic},
\cite{johnson2013accelerating}, \cite{defazio2014saga} turn out to be both very efficient
numerically (each update has a complexity comparable to vanilla SGD) and very sound theoretically,
because of strong linear convergence guarantees, that match or even improve the one of batch
solvers.
These algorithms involve gradient steps on the smooth part of the objective and theoretical
guarantees justify such steps under the gradient-Lipschitz assumptions thanks to the descent lemma
\cite[Proposition~A.24]{bertsekas1999nonlinear}.
Without this assumption, such theoretical guarantees fall apart.
Also, stochastic algorithms loose their numerical efficiency if their iterates
are projected on the feasible set $\Pi(X)$ at each iteration
as Equation~\eqref{eq:primal} requires.
STORC \cite{hazan2016variance} can deal with constrained objectives without
a full projection but is restricted to compact sets of constraints which
is not the case of $\Pi(X)$.
Then, a modified proximal gradient method from \cite{tran2015composite} provides convergence
bounds relying on self-concordance \cite{nesterov2013introductory} rather than the
gradient-Lipschitz property.
However, the convergence rate is guaranteed only once the iterates are close to the optimal
solution and we observed in practice that this algorithm is simply not working (since it ends up
using very small step-sizes) on the problems considered here.
Recently, \cite{lu2018relatively} has provided new descent lemmas based on
\emph{relative-smoothness} that hold on a wider set of functions including Poisson regression
losses.
This work is an extension of \cite{bauschke2016descent} that presented the same algorithm and
detailed its application to Poisson regression losses.
While this is more generic than our work, they only manage to reach sublinear convergence rates
$\mathcal{O}(1 / t)$ that applies only on positive solution (namely $w^* \in [0, \infty)^d$) while
we reach linear rates for any solution $w^* \in \R^d$.

\paragraph{Our contribution.} % (fold)
\label{par:our_contribution}

The first difficulty with the objective~\eqref{eq:primal} is to remain in the open
polytope $\Pi(X)$.
To deal with simpler constraints we rather perform optimization on the dual problem
\begin{equation}
  \label{eq:general_dual_problem}
  \max_{\alpha \in \Dfsm^n} D(\alpha)
  \quad \text{where} \quad
  D(\alpha) = \frac{1}{n} \sum_{i=1}^n - f^*_i (-\alpha_i)
        - \lambda g^* \Bigg(\frac{1}{\lambda n} \sum_{i=1}^n \alpha_i x_i -
        \frac{1}{\lambda} \psi \Bigg),
\end{equation}
where for a function $h$, the Fenchel conjugate $h^*$ is given by
$h^*(v) = \sup_u uv - h(u)$, and $\Dfsm$ is the domain of the function
$x \mapsto \sum_{i=1}^n f_i^*(-x)$.
This strategy is the one used by Stochastic Dual Coordinate Ascent
(SDCA) \cite{shalev2013stochastic}.
The dual problem solutions are box-constrained to $\Dfsm^n$ which is much easier
to maintain than the open polytope $\Pi(X)$.
Note that as all $f_i$ are strictly decreasing (because they are strictly monotone on $(0, +\infty)$
with $\lim_{t\rightarrow0} f_i(t)= +\infty$), their dual are defined on $\Dfs \subset (-\infty, 0)$.
By design, this approach keeps the dual constraints maintained all along the iterations
and the following proposition, proved in Section~\ref{sec:duality}, ensures that the primal iterate
converges to a point of $\Pi(X)$.
\begin{proposition}
\label{prop:strong-duality}
  Assume that the polytope $\Pi(X)$ is non-empty,
  the functions $f_i$ are convex, differentiable, with
  $\lim_{t\rightarrow0} f_i(t)= +\infty$ for $i=1, \dots, n$ and that $g$ is strongly convex.
  Then, strong duality holds, namely $P(w^*) = D(\alpha^*)$ and the Karush-Kuhn-Tucker conditions
  relate the two optima as
  \begin{equation*}
    \forall i \in \{ 1, \dots, n\}, \;
    \alpha_i^* = - {f_i^*}'({w^*}^\top x_i)
    \quad \text{and} \quad
    w^* = \nabla g^*\bigg(
      \frac{1}{\lambda n} \sum_{i=1}^n \alpha_i x_i -\frac{1}{\lambda} \psi \bigg),
  \end{equation*}
  where $w^* \in \Pi(X)$ is the minimizer of $P$ and $\alpha^* \in \Dfsm^n$ the maximizer of $D$.
\end{proposition}

In this chapter, we introduce the $\log$ smoothness property and its characteristics and then
we derive linear convergence rates for SDCA \emph{without the gradient-Lipschitz}
assumption, by replacing it with $\log$ smoothness, see Definition~\ref{hyp:new-gradient-assumption}.
Our results provide a state-of-the-art optimization technique for the considered
problem~\eqref{eq:primal}, with sound theoretical guarantees (see Section~\ref{sec:algorithm})
and very strong empirical properties as illustrated on experiments conducted with several datasets
for Poisson regression and Hawkes processes likelihood (see Section~\ref{sec:experiments}).
We study also SDCA with importance sampling \cite{zhao2015stochastic} under $\log$ smoothness and
prove that it improves both theoretical guarantees and convergence rates observed in practical
situations, see Sections~\ref{sub:importance_sampling} and~\ref{sec:experiments}.
We provide also a heuristic initialization technique in Section~\ref{sub:heuristic_for_a_wise_start}
and a "mini-batch" \cite{qu2016sdna} version of the algorithm
in Section~\ref{sub:optimizing_over_several_indices} that allows to end up with a particularly
efficient solver for the considered problems.
We motivate even further the problem considered in this chapter in
Figure~\ref{fig:poisson_toy_example}, where we consider a toy Poisson regression problem
(with 2 features and 3 data points), for which L-BFGS-B typically fails while SDCA works.
This illustrates the difficulty of the problem even on such an easy example.

\begin{figure}
\centering
\includegraphics[width=\textwidth]{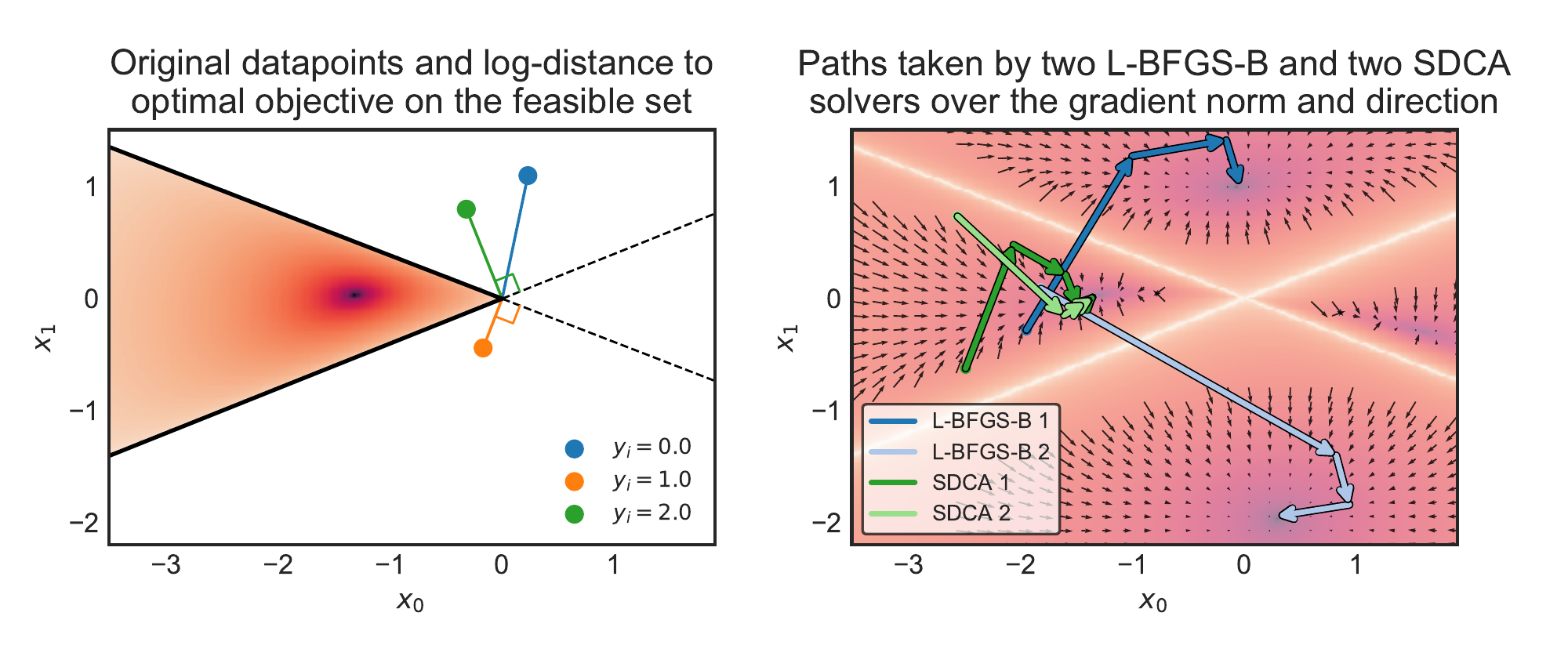}
\caption{Iterates of SDCA and L-BFGS-B on a Poisson regression toy example with three samples
  and two features.
  \emph{Left.} Dataset and value of the objective.
  \emph{Right:} Iterates of L-BFGS-B and SDCA with two different starting points.
  The background represents the gradient norm and the arrows the gradient direction.
  SDCA is very stable and converges quickly towards the optimum, while L-BFGS-B easily converges
  out of the feasible space.
  }
\label{fig:poisson_toy_example}
\end{figure}

\paragraph{Outline.} % (fold)
\label{par:outline}

% paragraph outline (end)

We first introduce the $\log$ smoothness property in Section~\ref{sec:log-smoothness},
relate it to self-concordance in Proposition~\ref{prop:log-smooth-second-order} and
translate it in the Fenchel conjugate space in Proposition~\ref{prop:hyp_dual_hessian_equivalence}.
Then, we present the shifted SDCA algorithm in Section~\ref{sec:algorithm} and state its convergence
guarantees in Theorem~\ref{th:general-convergence} under the $\log$ smoothness assumption.
We also provide theoretical guarantees for variants of the algorithm, one using proximal operators
\cite{shalev2014accelerated} and the second using importance sampling \cite{zhao2015stochastic}
which leads to better convergence guarantees in Theorem~\ref{th:convergence-is}.
In Section~\ref{sec:applications} we focus on two specific problems, namely Poisson
regression and Hawkes processes, and explain how they fit into the considered setting.
Section~\ref{sec:experiments} contains experiments that illustrate the benefits of our approach
compared to baselines.
This Section also proposes a very efficient heuristic initialization and numerical details
allowing to optimize over several indices at each iteration, which is a trick to accelerate even
further the algorithm.

\section{A tighter smoothness assumption}
\label{sec:log-smoothness}

To have a better overview of what $\log$ smoothness is, we formulate the following proposition
giving an equivalent property for $\log$ smooth functions that are twice differentiable.
\begin{proposition}
\label{prop:log-smooth-second-order}
Let $f : \Df \subset \R \rightarrow \R$ be a convex strictly monotone and twice
differentiable function. Then,
\begin{equation*}
  f \text{ is $L$-$\log$ smooth}
  \quad \Leftrightarrow \quad
  \forall x \in \Df, \; f''(x) \leq \tfrac{1}{L} f'(x)^2.
\end{equation*}
\end{proposition}
This proposition is proved in Section~\ref{sub:proof_of_proposition_prop:log-smooth-second-order}
and we easily derive from it that ${x \mapsto -L \log x}$ on $(0, +\infty)$, $x \mapsto L x$ on $\R$
and $x \mapsto L \exp(x)$ on $[0, +\infty)$ are $L$-$\log$ smooth.
This proposition is linked to the self-concordance property introduced by Nesterov
\cite{nesterov2013introductory} widely used to study losses involving logarithms.
For the sake of clarity, the results will be presented for functions whose domain $\Df$ is a
subset of $\R$ as this leads to lighter notations.
\begin{definition}
\label{def:self-concordance}
  A convex function $f : \Df \subset \R \rightarrow \R$ is standard self-concordant if
  \begin{equation*}
    \forall x \in \Df, \;
    | f'''(x) | \leq 2 f''(x)^{3/2}.
  \end{equation*}
\end{definition}
This property has been generalized \cite{bach2010self,sun2017generalized} but always consists in
controlling the third order derivative by the second order derivative, initially to bound the second
order Taylor approximation used in the Newton descent algorithm \cite{nesterov2013introductory}.
While right hand sides of both properties ($f'(x)^2$ and $2 f''(x)^{3/2}$) might look arbitrarily
chosen, in fact they reflect the motivating use case of the logarithmic barriers where
$f : t \mapsto - \log (t)$ and for which the bound is reached.
Hence, $\log$ smoothness is the counterpart of self-concordance but to control the second order
derivative with the first order derivative.
As it is similarly built, $\log$ smoothness shares the affine invariant property with
self-concordance.
It means that if $f_1$ is $L$ $\log$-smooth then $f_2: x \mapsto a x + b$ with $a, b \in \R$ is also
$L$-$\log$ smooth with the same constant $L$.
An extension to the multivariate case where $\Df \subset \R^d$ is likely feasible but is
useless for our algorithm and hence beyond the scope of this paper.
From the $\log$ smoothness property of a function $f$, we derive several characteristics for its
Fenchel conjugate $f^*$ starting with the following proposition.
\begin{proposition}
  \label{prop:hyp_dual_hessian_equivalence}
  Let $f : \Df \subset \R \rightarrow \R$ be a strictly monotone convex function and $f^*$ be its
  twice differentiable Fenchel conjugate.
  Then,
  \begin{equation*}
    f \text{ is $L$-$\log$ smooth}
    \quad \Leftrightarrow \quad
    \exists L > 0; \; \forall x \in \Dfs, \; {f^*}''(x) \geq L x^{-2}.
  \end{equation*}
\end{proposition}
This proposition is proved in Section~\ref{sec:proof_of_proposition_} and is the first step towards
a series of convex inequalities for the Fenchel conjugate of $\log$ smooth functions.
These inequalities, detailed in Section~\ref{sub:convex_inequality}, bounds the Bregman divergence
of such functions and are compared to what can be obtained with self-concordance or strong convexity
(on a restricted set) in the canonical case where $f : t \mapsto - \log (t)$.
It appears with $\log$ smoothness we obtain tighter bounds than what is achievable with other
assumptions and that all these bounds are reached (and hence cannot be improved) in the canonical
case (see Table~\ref{tab:convex-comparison}).

\section{The Shifted SDCA algorithm} % (fold)
\label{sec:algorithm}

The dual objective~\eqref{eq:general_dual_problem} cannot be written as a composite sum of convex
functions as in the general objective~\eqref{eq:general_primal}, which is required for stochastic
algorithms such as SRVG \cite{johnson2013accelerating} or SAGA \cite{defazio2014saga}.
It is better to use a coordinate-wise approach to optimize this problem, which leads to
SDCA \cite{shalev2014accelerated}, in which the starting point has been \emph{shifted}
by $\frac{1}{\lambda} \psi$.
This shift is induced by the relation linking primal and dual variables at the optimum: the second
Karush-Kuhn-Tucker condition from Proposition~\ref{prop:strong-duality},
\begin{equation}
\label{eq:general_primal_dual_relations_at_optimum_w}
  w^* = \nabla g^* \bigg(
          \frac{1}{\lambda n} \sum_{i=1}^n \alpha_i^* x_i -\frac{1}{\lambda} \psi
      \bigg).
\end{equation}
We first present the general algorithm (Algorithm~\ref{alg:general-shifted-sdca}), then its
proximal alternative (Algorithm~\ref{alg:prox-shifted-sdca}) and finally how importance sampling
leads to better theoretical results.
We assume that we know bounds $(\beta_i)_{1\leq i \leq n}$ such that $\beta_i / \alpha_i^* \geq 1$
for any $i=1, \ldots, n$,
such bounds can be explicitly computed from the data in the particular cases considered in this
chapter, see Section~\ref{sec:applications} for more details.

\begin{algorithm}[ht]
\caption{Shifted SDCA}
\begin{algorithmic}[1]
\Require{
  \begin{varwidth}[t]{\linewidth}
    Bounds $\beta_i \in \Dfsm$ such that
    $\forall i \in \{1, \dots, n\},~ \beta_i / \alpha_i^* \geq 1$, \\
    $\alpha^{(0)} \in \Dfsm^n$ dual starting point such that
    $\forall i \in \{1, \dots, n\},~ \beta_i / \alpha_i \geq 1$
 \end{varwidth}}
\State{$v^{(0)} = \frac{1}{\lambda n} \sum_{i=1}^n \alpha_i^{(0)} x_i -\frac{1}{\lambda} \psi$ }
\For{$t = 1, 2 \dots T$}
\State{Sample $i$ uniformly at random in $\{1, \ldots, n \}$}\label{lst:line:random}
\State{Find $\alpha_i$ that maximizes
    $- \frac{1}{n} f_i^*(-\alpha_i) - \lambda g^*\big(
      v^{(t - 1)} +
      (\lambda n)^{-1} (\alpha_i - \alpha_i^{(t-1)}) x_i
    \big)$} \label{lst:line:general_local_max}
\State{$\alpha_i \leftarrow \min(1, \beta_i / \alpha_i) \alpha_i$}
\label{lst:line:general-bounding}
\State{$\Delta \alpha_i \leftarrow \alpha_i - \alpha_i^{(t-1)}$}
\State{$\alpha^{(t)} \leftarrow \alpha^{(t - 1)} + \Delta \alpha_i e_i$}
\State{$v^{(t)} \leftarrow v^{(t - 1)} + (\lambda n)^{-1} \Delta \alpha_i x_i$}
\State{$w^{(t)} \leftarrow \nabla g^*( v^{(t)})$}
\EndFor
\end{algorithmic}
\label{alg:general-shifted-sdca}
\end{algorithm}
% This class of function include self-concordant functions such as ${-\log(x)}$,
% ${x \log(x) - \log(x)}$ or ${x \log(x) - x + \frac{1}{x}}$.
%
The next theorem provides a linear convergence rate for
Algorithm~\ref{alg:general-shifted-sdca} where we assume that each
$f_i$ is $L_i$-$\log$ smooth (see Definition~\ref{hyp:new-gradient-assumption}).
\begin{theorem}
\label{th:general-convergence}
Suppose that we known bounds $\beta_i \in \Dfsm$ such that
$R_i = \frac{\beta_i} {\alpha_i^*} \geq 1$ for $i=1,\dots,n$
and assume that all $f_i$ are $L_i$-$\log$ smooth with differentiable Fenchel conjugates
and $g$ is 1-strongly convex.
Then, Algorithm~\ref{alg:general-shifted-sdca} satisfies
\begin{equation}
\label{eq:general-convergence-rate}
\E [D(\alpha^*) - D(\alpha^{(t)})]
  \leq \Big( 1 - \frac{\min_{i} \sigma_i}{n} \Big)^t ( D(\alpha^*) - D(\alpha^{(0)})),
\end{equation}
where
\begin{equation}
\label{eq:sigma_i_def}
\sigma_i = \bigg( 1 + \frac{\|x_i\|^2  {\alpha_i^*}^2}{2 \lambda n L_i}
                      \frac{(R_i - 1)^2}
                           {\frac{1}{R_i} + \log R_i - 1}
           \bigg)^{-1}.
\end{equation}
\end{theorem}
The proof of Theorem~\ref{th:general-convergence} is given in Section~\ref{sub:contraction}.
It states that in the considered setting, SDCA achieves a linear convergence rate for the dual
objective.
The bounds $\beta_i$ are provided in Section~\ref{sec:applications} below for two particular cases:
Poisson regression and likelihood Hawkes processes.
Equipped with these bounds, we can compare the rate obtained in Theorem~\ref{th:general-convergence}
with already known linear rates for SDCA under the gradient-Lipschitz assumption
\cite{shalev2013stochastic}.
Indeed, we can restrict the domain of all $f_i^*$ to $(-\beta_i, 0)$
on which Proposition~\ref{prop:hyp_dual_hessian_equivalence} states that all $f_i$ are
$L_i / ({\alpha_i^*}^2 R_i^2)$-strongly convex.
Now, following carefully the proof in \cite{shalev2013stochastic} leads to the convergence rate
given in Equation~\eqref{eq:general-convergence-rate} but with
% Thus we can compare the rate obtained in Theorem~\ref{th:general-convergence} with already
% known linear rates for SDCA under the gradient-Lipschitz assumption (see
% \cite{shalev2013stochastic}).
\begin{equation*}
  \sigma_i = \bigg( 1 + \frac{\|x_i\|^2  {\alpha_i^*}^2} {\lambda n L_i}
  R_i^2 \bigg)^{-1}.
\end{equation*}
Since $2 \big(\frac{1}{R} + \log R - 1 \big) (R - 1)^{-2} \geq R^{-2}$ for any $R \geq 1$,
Theorem~\ref{th:general-convergence} provides a faster convergence rate.
The comparative gain depends on the values of $(\|x_i\|^2  {\alpha_i^*}^2) / (\lambda n L_i)$ and
$R_i$ but it increases logarithmically with the value of $R_i$.
Table~\ref{tab:convergence-rates} below compares the explicit values of these linear rates on a
dataset used in our experiments for Poisson regression.

\begin{remark}
\label{rmk:sdca-has-no-primal-path}
Convergence rates for the primal objective are not provided since the primal iterate $w^{(t)}$
typically belongs to $\Pi(X)$ only when it is close enough to the optimum.
This would make most of the values of the primal objective $P(w^{(t)})$ undefined and therefore not
comparable to $P(w^*)$.
\end{remark}

\subsection{Proximal algorithm} % (fold)
\label{sub:proximal_algorithm}

Algorithm~\ref{alg:general-shifted-sdca} maximizes the dual over one coordinate at
Line~\ref{lst:line:general_local_max} whose solution might not be explicit and
requires inner steps to obtain $\alpha_i^{(t)}$.
But, whenever $g$ can be written as
\begin{equation}
  \label{eq:g_ridge}
  g(w) = \tfrac{1}{2} \|w\|^2 + h(w),
\end{equation}
where $h$ is a convex, prox capable and possibly non-differentiable function,
we use the same technique as Prox-SDCA \cite{shalev2014accelerated} with a proximal lower bound
that leads to
\begin{equation*}
  \alpha_i^{(t)} =
      \argmax_{a_i \in \Dfsm}
      - f^*_i (-\alpha_i)
          - \frac{\lambda n}{2} \Big\| w^{(t-1)}
          - (\lambda n)^{-1} (\alpha_i - \alpha_i^{(t-1)}) x_i  \Big\|^2,
\end{equation*}
with
\begin{equation*}
  w^{(t-1)} =
    \prox_{h} \bigg(
      \frac{1}{\lambda n} \sum_{i=1}^n \alpha_i^{(t-1)} x_i -\frac{1}{\lambda} \psi
    \bigg),
\end{equation*}
see Section~\ref{sec:proximal_algorithm_supp} for details.
This leads to a proximal variant described in Algorithm~\ref{alg:prox-shifted-sdca} below, which
is able to handle various regularization techniques and which has the same convergence guarantees as
Algorithm~\ref{alg:general-shifted-sdca} given in Theorem~\ref{th:general-convergence}.
Also, note that assuming that $g$ can be written as~\eqref{eq:g_ridge} with a prox-capable function
$h$ is rather unrestrictive, since one can always add a ridge penalization term in the objective.
\begin{algorithm}[ht]
\caption{Shifted Prox-SDCA}
\begin{algorithmic}[1]
\Require{
  \begin{varwidth}[t]{\linewidth}
    Bounds $\beta_i \in \Dfsm$ such that
    $\forall i \in \{1, \dots, n\},~ \beta_i / \alpha_i^* \geq 1$, \\
    $\alpha^{(0)} \in \Dfsm^n$ dual starting point such that
    $\forall i \in \{1, \dots, n\},~ \beta_i / \alpha_i \geq 1$
 \end{varwidth}}
\State{$v^{(0)} = \frac{1}{\lambda n} \sum_{i=1}^n \alpha_i^{(0)} x_i -\frac{1}{\lambda} \psi$}
\For{$t = 1, 2 \dots T$}
\State{Sample $i$ uniformly at random in $\{1, \ldots, n \}$}\label{lst:line:random_sampling}
\State{Find $\alpha_i$ that maximize
       $- \frac{1}{n} f_i^*(-\alpha_i)
       - \frac{\lambda}{2} \Big\|
           w^{(t-1)} + (\lambda n)^{-1} (\alpha_i - \alpha_i^{(t - 1)}) x_i
       \Big\|^2$}
       \label{lst:line:local_max}
\State{$\alpha_i \leftarrow \min(1, \beta_i / \alpha_i) \alpha_i$}
\label{lst:line:bounding}
\State{$\Delta \alpha_i \leftarrow \alpha_i - \alpha_i^{(t-1)}$}
\State{$\alpha^{(t)} \leftarrow \alpha^{(t - 1)} + \Delta \alpha_i e_i$}
\State{$v^{(t)} \leftarrow v^{(t - 1)} + (\lambda n)^{-1} \Delta \alpha_i x_i$}
\State{$w^{(t)} \leftarrow \prox_h(v^{(t)})$}
\EndFor
\end{algorithmic}
\label{alg:prox-shifted-sdca}
\end{algorithm}

\subsection{Importance sampling} % (fold)
\label{sub:importance_sampling}

Importance sampling consists in adapting the probabilities of choosing a sample $i$ (which is
by default done uniformly at random, see Line~\ref{lst:line:random} from
Algorithm~\ref{alg:general-shifted-sdca}) using the improvement which is expected by sampling it.
% In SDCA case, roughly if we sample index $i$, the bigger the bound $\beta_i$ is, the
% smaller is the expected gain.
% We might then adapt the sampling probabilities to obtain more often indices $i$ associated to
% lower bounds $\beta_i$.
Consider a distribution $\rho$ on $\{ 1, \dots, n \}$ with probabilities
$\{ \rho_1, \dots, \rho_n \}$ such that $\rho_i \geq 0$ for any $i$ and $\sum_{i=1}^n \rho_i = 1$.
The Shifted SDCA and Shifted Prox-SDCA with importance sampling algorithms are simply obtained by
modifying the way $i$ is sampled in Line~\ref{lst:line:random_sampling} of
Algorithms~\ref{alg:general-shifted-sdca} and~\ref{alg:prox-shifted-sdca}: instead of sampling
uniformly at random, we sample using such a distribution~$\rho$.
The optimal sampling probability $\rho$ is obtained in the same way as \cite{zhao2015stochastic}
and it also leads under our $\log$ smoothness assumption to a tighter convergence rate, as
stated in Theorem~\ref{th:convergence-is} below.
\begin{theorem}
\label{th:convergence-is}
Suppose that we known bounds $\beta_i \in \Dfsm$ such that
$R_i = \frac{\beta_i} {\alpha_i^*} \geq 1$ for $i=1,\ldots,n$
and assume that all $f_i$ are $L_i$-$\log$ smooth with differentiable Fenchel conjugates and
$g$ is 1-strongly convex.
Consider $\sigma$ defined by~\eqref{eq:sigma_i_def} and consider the distribution
\begin{equation*}
  \rho_i = \frac{\sigma_i^{-1}}{\sum_{j=1}^{n}\sigma_j^{-1}}
\end{equation*}
for $i \in \{1, \ldots, n \}$.
Then, Algorithm~\ref{alg:general-shifted-sdca} and~\ref{alg:prox-shifted-sdca} where
Line~\ref{lst:line:random_sampling} is replaced by sampling $i \sim \rho$ satisfy
\begin{equation*}
    \E [D(\alpha^*) - D(\alpha^{(t)})]
      \leq \Big( 1 - \frac{\bar{\sigma}}{n} \Big)^t ( D(\alpha^*) - D(\alpha^{(0)}),
\end{equation*}
where $ \bar{\sigma} = \big(\frac{1}{n} \sum_{i=1}^{n} \sigma_i^{-1}\big)^{-1}$.
\end{theorem}
The proof is given in Section~\ref{sub:theorem_th:convergence-is}.
This convergence rate is stronger than the previous one from Theorem~\ref{th:general-convergence}
since $(\frac{1}{n} \sum_{i=1}^{n} \sigma_i^{-1})^{-1} \geq \min_i \sigma_i$.
Table~\ref{tab:convergence-rates} below compares the explicit values of these linear rates on a
dataset used in our experiments for Poisson regression (facebook dataset).
We observe that the $\log$ smooth rate with importance sampling is orders of magnitude better
than the one obtained with the standard theory for SDCA which exploits only the
$L_i / ({\alpha_i^*}^2 R_i^2)$ strong convexity of the functions $f_i^*$.
\begin{table}[htbp]
  \centering
  \begin{tabular}{cccc}
    \toprule
    strongly convex & strongly convex with & $\log$ smooth &  $\log$ smooth with \\
            &  importance sampling &      & importance sampling \\
    \midrule
    $(0.9999)^t$ & $(0.9969)^t$ & $(0.9984)^t$ &  {$(\boldsymbol{0.9679})^t$} \\
    \bottomrule
  \end{tabular}
  \caption{Theoretical convergence rates obtained on the facebook dataset
  (see Section~\ref{sub:poisson_regression_expe}) in four different settings: strongly convex
  (which is the rate obtained when all functions $f_i$ are considered
  $L_i / ({\alpha_i^*}^2 R_i^2)$-strongly convex) with and without
  importance sampling \cite{shalev2013stochastic,zhao2015stochastic}
  and the rate obtained in the setting considered in the chapter, with and
  without importance sampling.
  In this experiment, the maximum value for $R_i$ is 9062 and its average value is 308.
  As expected, the best rate is obtained by combining the $\log$-smoothness property with
  importance sampling.}
  \label{tab:convergence-rates}
\end{table}

\section{Applications to Poisson regression and Hawkes processes} % (fold)
\label{sec:applications}

In this Section we describe two important models that fit into the setting of this chapter.
We precisely formulate them as in Equation~\eqref{eq:primal} and give the explicit value of bounds
$\beta_i$ such as $\alpha_i^* \leq \beta_i$, where $\alpha^*$ is the solution to the dual problem
\eqref{eq:general_dual_problem}.

\subsection{Linear Poisson regression} % (fold)
\label{sub:linear-poisson-regression}

Poisson regression is widely used to model count data, namely when, in the dataset, each
observation $x_i \in \R^d$ is associated an integer output $y_i \in \N$ for $i = 1, \dots, n$.
It aims to find a vector $w \in \R^d$ such that for a given function
$\phi : \mathcal{D}_\phi \subset \R \rightarrow (0, +\infty)^+$,
$y_i$ is the realization of a Poisson random variable of intensity $\phi(w^\top x_i)$.
A convenient choice is to use $\exp$ for $\phi$ as it always guarantees that $\phi(w^\top x_i) > 0$.
However, using the exponential function assumes that the covariates have a multiplicative effect
that often cannot be justified.
The tougher problem of linear Poisson regression, where $\phi(t) = t$ and
$\mathcal{D}_\phi$ is the polytope $\Pi(X)$, appears to model additive effects.
For example, this applies in image reconstruction.
The original image is retrieved from photons counts $y_i$ distributed as a Poisson distribution with
intensity $w^\top x_i$, that are received while observing the image with different detectors
represented by the vectors $x_i \in \R^d$.
This application has been extensively studied in the literature, see
\cite{harmany2012spiral,bauschke2016descent,tran2015composite} and \cite{bertero2009image} for a
review with a hundred references.
Linear Poisson regression is also used in various fields such as survival analysis with additive
effects \cite{boshuizen2010fitting} and web-marketing \cite{chen2009large} where the intensity
corresponds to an intensity of clicks on banners in web-marketing.
To formalize, we consider a training dataset $(x_1, y_1), \ldots, (x_{n_0}, y_{n_0})$ with
$x_i \in \R^d$ and $y_i \in \N$ and assume without loss of generality that $y_i > 0$ for
$i \in \{ 1, \ldots, n \}$ while $y_i = 0$ for $i \in \{ n + 1, \ldots, n_0 \}$ where
$n = \# \{ i : y_i > 0 \} \leq n_0$
(this simply means that we put first the samples corresponding to a label $y_i > 0$).
The negative log-likelihood of this model with a penalization function $g$ can be written as
\begin{equation*}
  P_0(w) = \frac{1}{n_0} \sum_{i=1}^{n_0} ( w^\top x_i - y_i \log(w^\top x_i))
  + \lambda_0 g(w)
\end{equation*}
where $\lambda_0 > 0$ corresponds to the level of penalization, with the constraint that
$w^\top x_i$ for $i=1, \ldots, n$.
This corresponds to Equation~\eqref{eq:primal} with $f_i(w) = -y_i \log(x_i^\top w)$ for
$i=1, \ldots, n$, which are $y_i$-$\log$ smooth functions, and with
\begin{equation*}
\psi = \frac{1}{n}\sum_{i=1}^{n_0} x_i \quad \text{and} \quad
\lambda = \frac{n_0}{n} \lambda_0.
\end{equation*}
Note that the zero labeled observations can be safely removed from the sum and are fully encompassed
in $\psi$.
The algorithms and results proposed in Section~\ref{sec:algorithm} can therefore be applied for
this model.

\subsection{Hawkes processes}

Hawkes processes are used to study cross causality that might occur in one or several events
series.
First, they were introduced to study earthquake propagation, the network across which the
aftershocks propagate can be recovered given all tremors timestamps \cite{ogata1999seismicity}.
Then, they have been used in high frequency finance to describe market reactions to different
types of orders  \cite{bacry2015hawkes}.
In the recent years Hawkes processes have found many new applications including crime
prediction \cite{mohler2013modeling} or social network information
propagation \cite{lukasik2016hawkes}.
A Hawkes process \cite{hawkes1974cluster} is a multivariate point-process:
it models timestamps $\{t_k^i\}_{i \geq 1}$ of nodes
$i=1, \ldots, I$ using a multivariate counting process with a particular auto-regressive
structure in its intensity.
More precisely, we say that a multivariate counting process
$N_t = [N_t^1, \ldots, N_t^I]$ where
$N_t^i = \sum_{k \geq 1} \ind{t_k^i \leq t}$ for $t \geq 0$ is a Hawkes process if the intensity
of $N^i$ has the following structure:
\begin{equation*}
\lambda^i(t) = \mu_i + \sum_{j=1}^I \int \phi_{ij}(t - s) \dif N^j(s)
             = \mu_i + \sum_{j=1}^I \sum_{k \; : \; t_k^j < t} \phi_{ij}(t - t_k^j).
\end{equation*}
The $\mu_i \geq 0$ are called \emph{baselines} intensities, and correspond to the exogenous
intensity of events from node $i$, and the functions $ \phi_{ij}$ for $1 \leq i, j \leq I$ are
called \emph{kernels}.
They quantify the influence of past events from node $j$ on the intensity of events from node $i$.
The main parametric model for the kernels is the so-called \emph{exponential} kernel, in which
we consider
\begin{equation}
  \label{eq:sum_of_exponentials_kernel}
  \phi^{ij}(t) = \sum_{u=1}^U a_u^{ij} b_u \exp (- b_u t)
\end{equation}
with $b_u > 0$.
In this model the matrix $A = [\sum_{u=1}^U a_u^{ij}]_{1 \leq i, j \leq d}$ is understood as an
\emph{adjacency matrix}, since entry $A_{i, j}$ quantifies the impact of the activity of node $j$
on the activity of node $i$, while $b_u > 0$ are memory parameters.
We stack these parameters into a vector $\theta$ containing the baselines $\mu_i$ and the self and
cross-excitation parameters $a_u^{ij}$.
Note that in this model the memory parameters $b_u$ are supposed to be given.
The associated goodness-of-fit is the negative log-likelihood, which is given by the general
theory of point processes (see \cite{daley2007introduction}) as
\begin{equation*}
-\ell(\theta) = - \sum_{i=1}^I \ell_i(\theta),
\quad \text{with} \quad
- \ell_i(\theta) = \int_0^T \lambda_\theta^i(t) dt - \int_0^T \log (\lambda_\theta^i(t))\dif N^i(t).
\end{equation*}
Let us define the following weights for $i, j = 1, \dots I$ and $u = 1, \dots, U$,
\begin{equation}
  \label{eq:hawkes_precom_weights}
  g^{j}_u(t) = \sum_{k \; : \; t_k^j < t} b_u e^{-b_u (t - t_k^j)},
  \quad g^{ij}_{u, k} = g^{j}_u(t_k^i) \quad \text{and} \quad
G^{j}_u = \int_0^T g^{j}_u(t) dt
\end{equation}
that can be computed efficiently for exponential kernels thanks to recurrence formulas (the
complexity is linear with respect to the number of events of each node).
Using the parametrization of the kernels from Equation~\eqref{eq:sum_of_exponentials_kernel} we can
rewrite each term of the negative log-likelihood as
\begin{equation*}
    - \ell_i(\mu_i, a^i)
    = \sum_{i=1}^I \bigg[  \mu^i T + \sum_{j=1}^I \sum_{u=1}^U a^{ij}_u G_u^j
         - \sum_{k=1}^{n_i} \log \Big(
              \mu^i + \sum_{j=1}^I \sum_{u=1}^U a^{ij}_u g^{ij}_{u, k}
            \Big) \bigg].
\end{equation*}
To rewrite $\ell_i$ in a vectorial form we define
$n_i$ as the number of events of node $i$ and
the following vectors for $i=1, \dots, I$:
\begin{equation*}
w^i =
\begin{bmatrix}
  \mu^i & a^{i,1}_1 & \cdots & a^{i,1}_U & \cdots & a^{i, I}_1 &
  \cdots & a^{i, I}_U
\end{bmatrix}^\top,
\end{equation*}
that are the model weights involved in $\ell_i$, and
\begin{equation*}
% \label{eq:hawkes-psi}
\psi^i = \frac{1}{n_i}
\begin{bmatrix}
   T & G^1_1 & \cdots G^1_U & \cdots & G^I_1 & \cdots & G^I_U \\
\end{bmatrix}^\top,
\end{equation*}
which correspond to the vector involved in the linear part of the primal objective
\eqref{eq:primal} and finally
\begin{equation*}
  % \label{eq:hawkes_feature_definition}
  x^i_k =
  \begin{bmatrix}
    1 & g^{i, 1}_{1, k} & \cdots & g^{i, 1}_{U, k} & \cdots &
    g^{i, I}_{1, k} & \cdots & g^{i, I}_{U, k} \\
  \end{bmatrix}^\top,
\end{equation*}
for $k = 1, \dots, n_i$ which contains all the timestamps data computed in the weights computed
in Equation~\eqref{eq:hawkes_precom_weights}.
With these notations the negative log-likelihood for node $i$ can be written as
\begin{equation*}
  -\ell(w) = - \sum_{i=1}^I \ell_i(w^i) \quad \text{with}
  \quad - \tfrac{1}{n_i} \ell_i(w^i) = (w^i)^\top \psi^i - \frac{1}{n_i} \sum_{k=1}^{n_i}
  \log ( (w^i)^\top x_k^i).
\end{equation*}
First, it shows that the negative log-likelihood can be separated into $I$ independent sub-problems
with goodness-of-fit $-\ell_i(w^i)$ that corresponds to the intensity of node $i$ with the
weights $x_{i, k}$ carrying data from the events of the other nodes $j$.
Each subproblem is a particular case of the primal objective \eqref{eq:primal},
where all the labels $y_i$ are equal to $1$.
As a consequence, we can use the algorithms and results from Section~\ref{sec:algorithm} to train
penalized multivariate Hawkes processes very efficiently.

\subsection{Closed form solution and bounds on dual variables}

In this Section with provide the explicit solution to Line~\ref{lst:line:local_max} of
Algorithm~\ref{alg:prox-shifted-sdca} when the objective corresponds to the linear Poisson
regression or the Hawkes process goodness-of-fit.
In Proposition~\ref{prop:closed_form_solution_log_losses} below we provide the
closed-form solution of the local maximization step corresponding to Line~\ref{lst:line:local_max}
of Algorithm~\ref{alg:prox-shifted-sdca}.

\begin{proposition}
\label{prop:closed_form_solution_log_losses}
  For Poisson regression and Hawkes processes, Line~\ref{lst:line:local_max} of
  Algorithm~\ref{alg:prox-shifted-sdca} has a closed form solution, namely
  \begin{equation*}
  \alpha^t_i = \frac{1}{2} \Bigg(
           \sqrt{
              \Big(
                  \alpha_i^{(t - 1)}
                  - \frac{\lambda n}{\|x_i\|^2} x_i^\top w^{(t-1)} \Big)^2
              + 4 \lambda n \frac{y_i}{\|x_i\|^2}
            }
            + \alpha_i^{(t - 1)}
            - \frac{\lambda n}{\|x_i\|^2} x_i^\top w^{(t-1)} \Bigg).
  \end{equation*}
\end{proposition}
This closed-form expression allows to derive a numerically very efficient training algorithm, as
illustrated in Section~\ref{sec:experiments} below.
For these two use cases, the dual loss is given by
$f_i^*(- \alpha_i) = - y_i -y_i \log (\frac{\alpha_i}{y_i})$ for any $\alpha_i > 0$
(with $y_i = 1$ for the Hawkes processes).
For this specific dual loss, we can provide also upper bounds $\beta_i$ for
all optimal dual variables $\alpha_i^*$, as stated in the next Proposition.
\begin{proposition}
\label{prop:dual_large_bound}
For Poisson regression and Hawkes processes, if $g(w) = \tfrac{1}{2} \|w\|^2$
and if ${x_i}^\top x_j \geq 0$ for all $1 \leq i, j \leq n$,
we have the following upper bounds on the dual variables at the optimum\textup:
\begin{equation*}
    \alpha_i^* \leq \beta_i
    \quad  \text{where} \quad
    \beta_i = \frac{1}{2 \|x_i\|^2} \bigg( n \psi^\top x_i
            + \sqrt{ ( n \psi^\top x_i )^2 +
                     4 \lambda n y_i \|x_i\|^2 }\bigg)
\end{equation*}
for any $i = 1, \ldots, n$.
\end{proposition}
The proofs of Propositions~\ref{prop:closed_form_solution_log_losses}
and~\ref{prop:dual_large_bound} are provided in Section~\ref{sub:closed_for_solution_for_log_losses}.
Note that the inner product assumption ${x_i}^\top x_j \geq 0$ from Proposition
\ref{prop:dual_large_bound} is mild: it is always met for the Hawkes process with kernels given
by~\eqref{eq:sum_of_exponentials_kernel} and it it met for Poisson regression whenever one applies
for instance a min-max scaling on the features matrix.

% \begin{remark}
%   The order of $\beta_i$ is $n \tfrac{\psi^\top x_i}{\|x_i\|^2}$ as we can expect
%   $n \psi^\top x_i \gg 2 \sqrt{\lambda n y_i \|x_i\|^2}$.
%   This bound is so large because it is very complicated to quantify the positive quantity
%   $\sum_{j\neq i} \alpha_j^* x_i^\top x_j$.
%   These $n-1$ positive terms are then removed from inequality while they should be retrieved from
%   $n \psi^\top x_i$.
% \end{remark}
%
\begin{remark}
  \label{rmk:closed-form-bounded}
  The closed form solution from Proposition~\ref{prop:closed_form_solution_log_losses} is always
  lower than the generic bound $\beta_i$, as explained in
  Section~\ref{sec:proof_of_remark_rmk:closed-form-bounded}.
  Hence, we actually do not need to manually bound $\alpha_i^{(t)}$ at
  line~\ref{lst:line:general-bounding}
  of Algorithm~\ref{alg:general-shifted-sdca} in this particular case.
\end{remark}

% subsection specification_for_log_losses (end)

% subsection hawkes_processes (end)

% section applications (end)

\section{Experiments} % (fold)
\label{sec:experiments}

To evaluate efficiently Shifted SDCA we have compared it with other optimization algorithms that
can handle the primal problem~\eqref{eq:primal} nicely, without the gradient-Lipschitz assumptions.
We have discarded the modified proximal gradient method
from \cite{tran2015composite} since most of the time it was diverging while computing the initial
step with the Barzilai-Borwein method on the considered examples.
We consider the following algorithms.

\paragraph{NoLips.} % (fold)
This is a first order batch algorithm that relies on relative-smoothness \cite{lu2018relatively}
instead of the gradient Lipschitz assumption.
Its application to linear Poisson regression has been detailed in \cite{bauschke2016descent} and its
analysis provides convergence guarantees with a sublinear convergence rate in $\mathcal{O}(1/n)$.
However, this method is by design limited to solutions with positive entries (namely
$w^* \in [0, \infty)^d$) and provides guarantees only in this case.
Its theoretical step-size decreases linearly with $1/n$ and is too small in practice.
Hence, we have tuned the step-size to get the best objective after 300 iterations.

\paragraph{SVRG.} % (fold)

This is a stochastic gradient descent algorithm with variance reduction introduced
in \cite{johnson2013accelerating, xiao2014proximal}.
We used a variant introduced in \cite{tan2016barzilai}, which uses Barzilai-Borwein in order to
adapt the step-size, since gradient-Lipschitz constants are unavailable in the considered setting.
We consider this version of variance reduction, since alternatives such as
SAGA\cite{defazio2014saga} and SAG \cite{schmidt2013minimizing} do not propose variants with
Barzilai-Borwein type of step-size selection.

\paragraph{L-BFGS-B.} % (fold)

L-BFGS-B is a limited-memory quasi-Newton algorithm \cite{nocedal1980updating,
nocedal2006nonlinear}.
It relies on an estimation of the inverse of the Hessian based on gradients differences.
This technique allows L-BFGS-B to consider the curvature information leading to faster convergence
than other batch first order algorithms such as ISTA and FISTA \cite{beck2009fast}.

\paragraph{Newton algorithm.} % (fold)

This is the standard second-order Newton algorithm which computes at each iteration the hessian of
the objective to solve a linear system with it.
In our experiments, the considered objectives are both $\log$-smooth and self-concordant
\cite{nesterov2013introductory}.
The self-concordant property bounds the third order derivative by the second order derivative,
giving explicit control of the second order Taylor expansion \cite{bach2010self}.
This ensures supra-linear convergence guarantees and keeps all iterates in the open
polytope~\eqref{eq:feasible_polytope} if the starting point is in it \cite{nesterov1994interior}.
However, the computational cost of the hessian inversion makes this algorithm scale very poorly with
the number of dimensions $d$ (the size of the vectors $x_i$).

\paragraph{SDCA.}

This is the Shifted-SDCA algorithm, see Algorithm~\ref{alg:prox-shifted-sdca}, without importance
sampling.
Indeed, the bounds given in Proposition~\ref{prop:dual_large_bound} are not tight enough to improve
convergence when used for importance sampling in the practical situations considered in this
Section (despite the fact that the rates are theoretically better).
A similar behavior was observed in~\cite{priol2017adaptive}.

\medskip
SVRG and L-BFGS-B are almost always diverging in these experiments just like in the simple
example considered in Figure~\ref{fig:poisson_toy_example}.
Hence, the problems are tuned to avoid any violation of the open polytope
constraint~\eqref{eq:feasible_polytope}, and to output comparable results between algorithms.
Namely, to ensure that $w^\top x_i> 0$ for any iterate $w$, we scale the vectors $x_i$ so that
they contain only non-negative entries, and the iterates of SVRG and L-BFGS-B are projected
onto $[0, +\infty)^{d}$.
This highlights two first drawbacks of these algorithms: they cannot deal with a generic feature
matrix and their solutions contain only non-negative coefficients.
For each run, the simply take $\lambda = \overline{x} / n$ where $
{\overline{x} = \tfrac{1}{n} \sum_{i=1}^n\|x_i\|^2}$.
This simple choice seemed relevant for all the considered problems.

\subsection{Poisson regression} % (fold)
\label{sub:poisson_regression_expe}

For Poisson regression we have processed our feature matrices to obtain coefficients between 0
and~1.
Numerical features are transformed with a min-max scaler and categorical features are one hot
encoded.
We run our experiments on six datasets found on UCI dataset repository
\cite{Lichman:2013} and Kaggle\footnote{\url{https://www.kaggle.com/datasets}} (see
Table~\ref{tab:poisson-datasets} for more details).
\begin{table}
  \caption{Poisson datasets details.}
  \label{tab:poisson-datasets}
  \centering
  \begin{tabular}{c|cccccc}
    \toprule
    dataset &
    wine \footnote{ \cite{cortez2009modeling}
    \url{https://archive.ics.uci.edu/ml/datasets/wine+quality}} &
    facebook \footnote{\cite{moro2016predicting}
    \url{https://archive.ics.uci.edu/ml/datasets/Facebook+metrics}} &
    vegas \footnote{\cite{moro2017stripping}
    \url{https://archive.ics.uci.edu/ml/datasets/Las+Vegas+Strip}} &
    news \footnote{\cite{fernandes2015proactive}
    \url{https://archive.ics.uci.edu/ml/datasets/online+news+popularity}}&
    property \footnote{
    \url{https://www.kaggle.com/c/liberty-mutual-group-property-inspection-prediction}} &
    simulated \footnote{The features have been generated with a folded normal distribution and the
    generating vector $w$ has 30 non zero coefficients sampled with a normal distribution.} \\
    \midrule
    \# lines & 4898 & 500 & 2215 & 504 & 50099 & 100000 \\
    \# features & 11 & 41 & 102 & 160 & 194 & 100 \\
    \bottomrule
  \end{tabular}
\end{table}
These datasets are used to predict a number of interactions for a social post (news
and facebook), the rating of a wine or a hotel (wine and vegas) or the number of hazards
occurring in a property (property).
The last one comes from simulated data which follows a Poisson regression.
In Figure~\ref{fig:poisson-real-data} we present the convergence speed of the five algorithms.
As our algorithms follow quite different schemes, we measure this speed regarding to the
computational time.
In all runs, NoLips, SVRG and L-BFGS-B cannot reach the optimal solution as the problem minimizer
contains negative values.
This is illustrated in detail in Figure~\ref{fig:poisson-vegas-stem}
for vegas dataset where it appears that all solvers obtain similar results for the positive values
of $w^*$ but only Newton and SDCA algorithms are able to estimate the negatives values of
$w^*$.
As expected, the Newton algorithm becomes very slow as the number of features $d$ increases.
SDCA is the only first order solver that reaches the optimal solution.
It combines the best of both world, the scalability of a first order solver and the ability to reach
solutions with negative entries.

\begin{figure}
\centering
\includegraphics[width=0.9\textwidth]{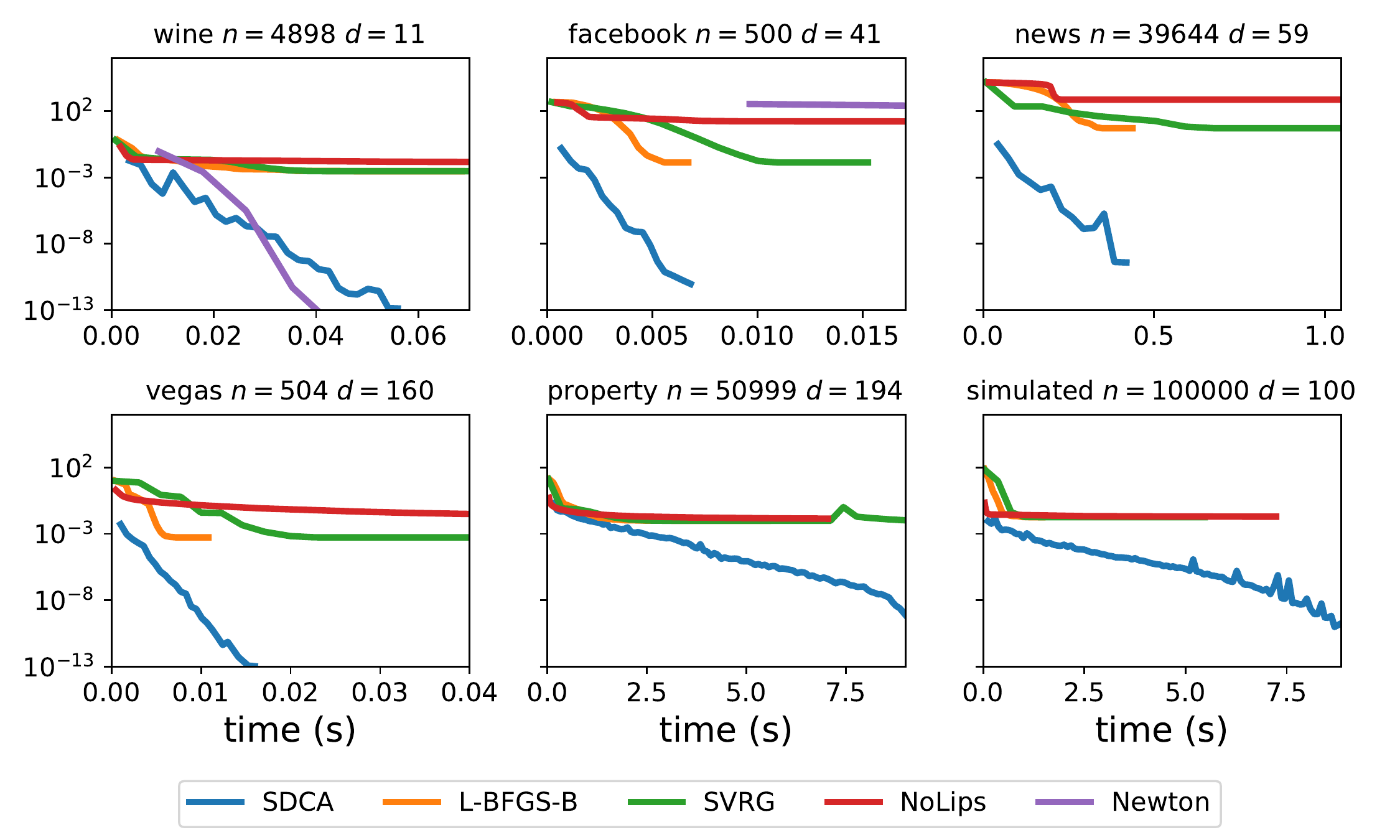}
\caption{Convergence over time of five algorithms SDCA, SVRG, NoLips, L-BFGS-B and Newton on~6
         datasets of Poisson regression.
         SDCA combines the best of both worlds: speed and scalability of SVRG and L-BFGS-B with the
         precision of Newton's solution.}
\label{fig:poisson-real-data}
\end{figure}

\begin{figure}
\centering
\includegraphics[width=0.8\textwidth]{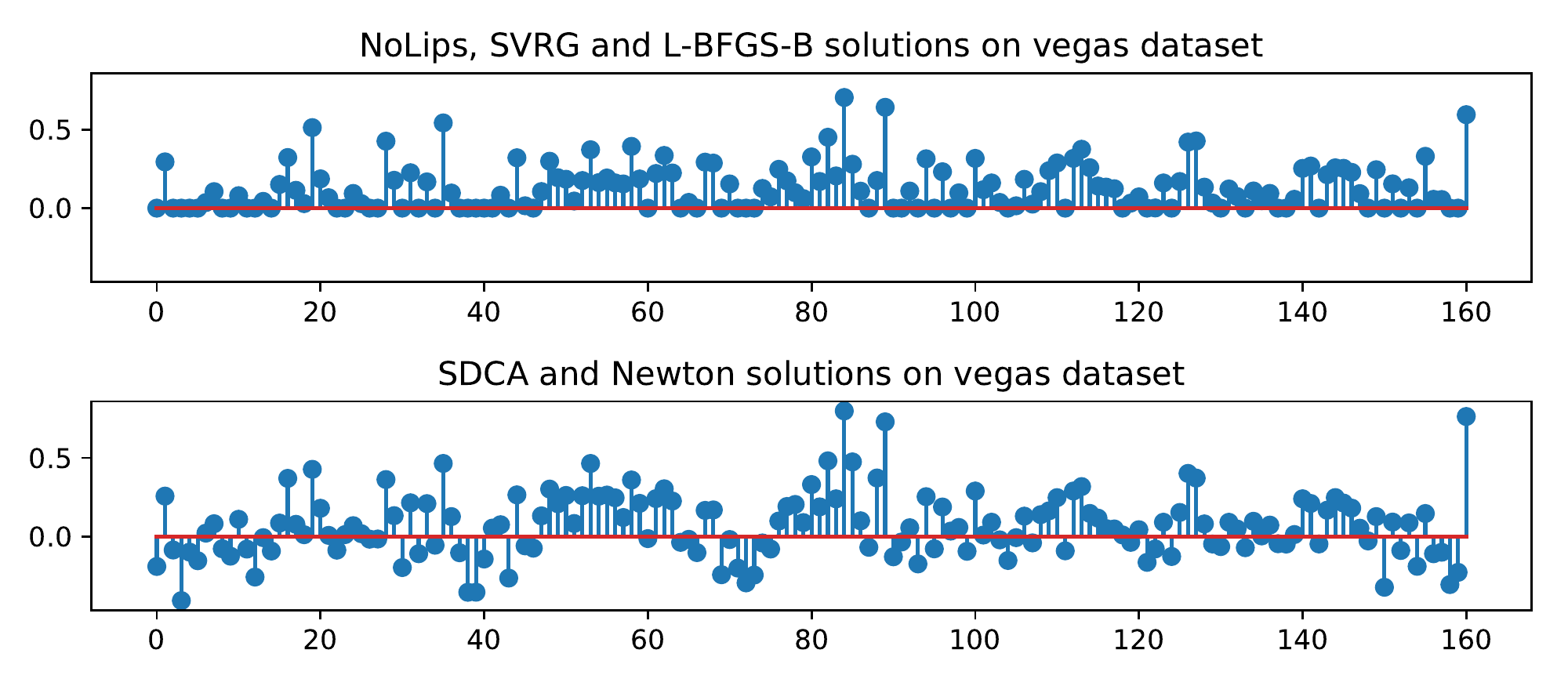}
\caption{Estimated minimizers $w^*$ on the vegas dataset (160 features).
The positive entries are roughly similarly recovered by all solvers but the negative entries are
only retrieved by SDCA and Newton algorithms.
% For example this allows the intercept (last index on the right) to be higher as it can be
% compensated by negative features.
}
\label{fig:poisson-vegas-stem}
\end{figure}

% subsection poisson_regression (end)

\subsection{Hawkes processes} % (fold)
\label{sub:hawkes_processes_exp}

If the adjacency matrix $A$ is forced to be entrywise positive, then no event type can have an
inhibitive effect on another.
% Neglecting negative effects in Hawkes processes would mean that we cannot discover any
% inhibitive effect of one event type over another.
%
This ability to exhibit inhibitive effect has direct implications on real life datasets especially
in finance where these effects are common
\cite{bacry2014estimation,bacry2015hawkes,rambaldi2017role}.
In Figure~\ref{fig:hawkes-bund} we present the aggregated influence of the kernels obtained after
training a Hawkes process on a finance dataset exploring market
microstructure \cite{bacry2015hawkes}.
While L-BFGS-B (or SVRG, or NoLips) recovers only excitation in the adjacency matrix,
SDCA also retrieves inhibition that one event type might have on another.
\begin{figure}
\centering
\includegraphics[width=0.7\textwidth]{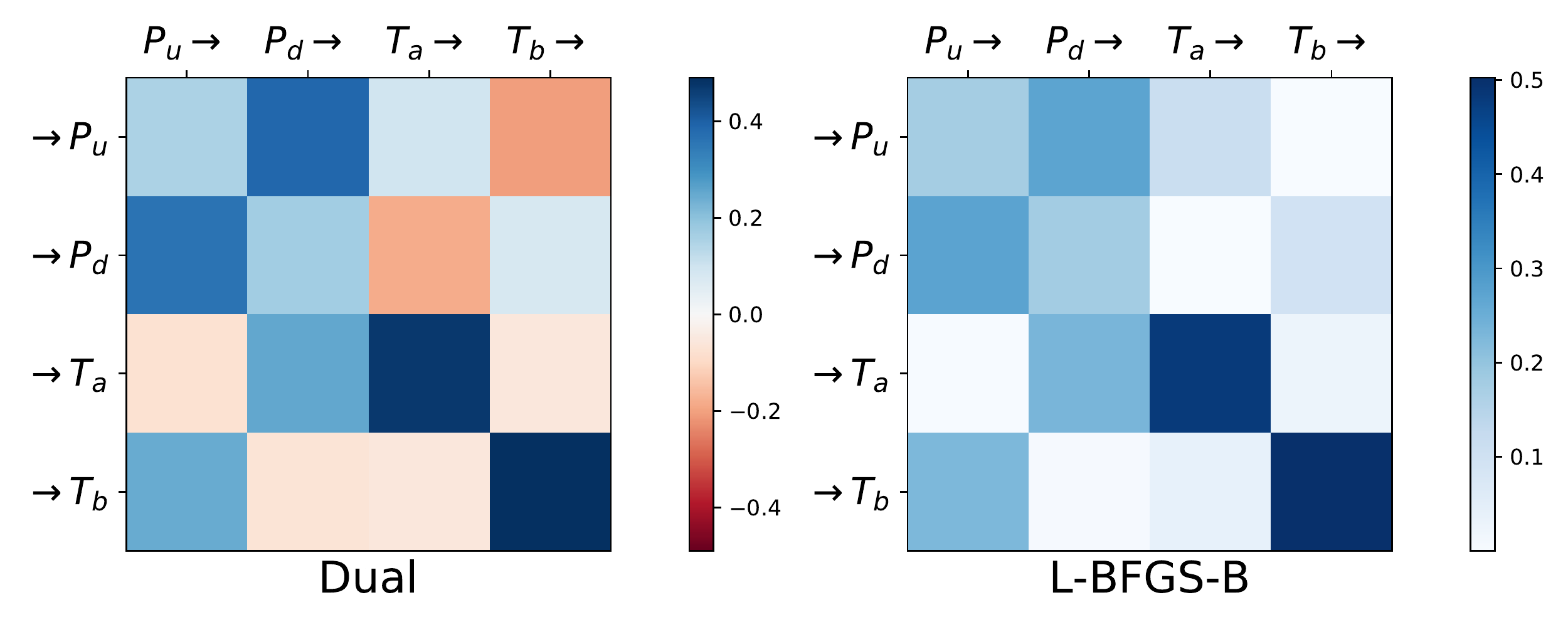}
\caption{
    Adjacency matrix $A$ of a Hawkes process fitted on high-frequency financial data from the Bund
    market.
    This reproduces an experiment run in \cite{bacry2015hawkes} where $P_u$ (resp. $P_d$)
    counts the number of upward (resp. downward) mid-price moves and $T_a$ (resp. $T_b$) counts the
    number of market orders at the best ask (resp. best bid) that do not move the price.
    SDCA detects inhibitive behaviors while L-BFGS-B cannot.
}
\label{fig:hawkes-bund}
\end{figure}
It is expected that when stocks are sold (resp. bought) the price is unlikely to go up (resp. down)
but this is retrieved by SDCA only.
% This shows how valuable it is to unset the positive constraint in order to detect inhibition.
On simulated data this is even clearer and in Figure~\ref{fig:hawkes-simulated-data} we observe
the same behavior when the ground truth contains inhibitive effects.
Our experiment consists in two simulated Hawkes processes with 10 nodes and sum-exponential
kernels with 3 decays.
There are only excitation effects - all $a^{ij}_u$ are positive
%for $i, j = 1, \ldots, n$, $u=1, \ldots I$
- in the first case and we allow inhibitive effects in the second.
Events are simulated according to these kernels that we try to recover.
While it would be standard to compare the performances in terms of log-likelihood
obtained on the a test sample, nothing ensures that the problem optimizer lies in the feasible set
of the test set.
Hence the results are compared by looking at the estimation error (RMSE) of the adjacency matrix
$A$ across iterations.
Figure~\ref{fig:hawkes-simulated-data} shows that SDCA always converges faster towards its solution
in both cases and that when the adjacency matrix contains inhibitive effects, SDCA obtains a better
estimation error than L-BFGS-B.
\begin{figure}
\centering
\includegraphics[width=0.7\textwidth]{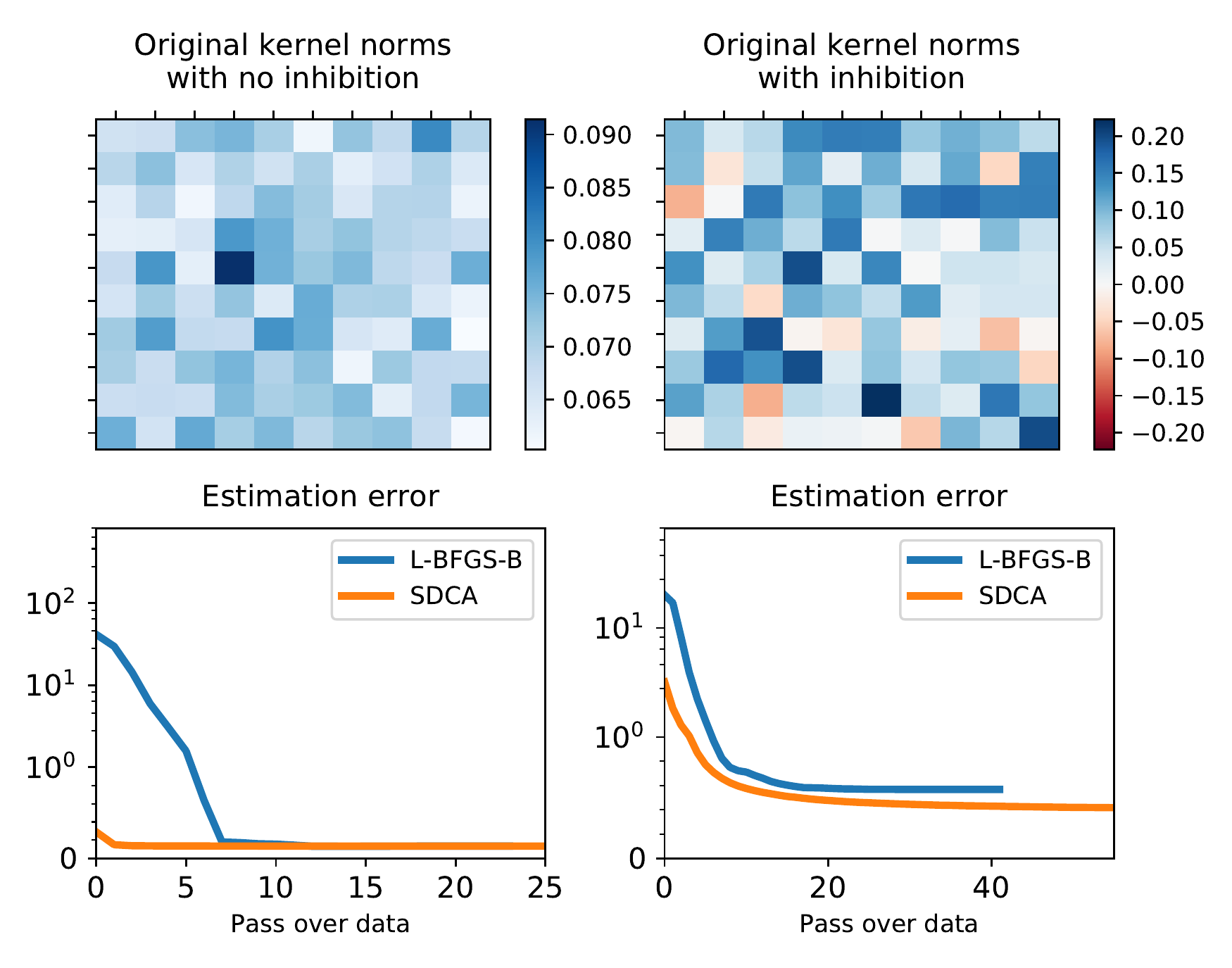}
\caption{
\emph{Top:} Adjacency matrix of the Hawkes processes used for simulation.
\emph{Bottom:} estimation error of the adjacency matrix $A$ across iterations.
In both case SDCA is faster than L-BFGS-B and reaches a better estimation error when there are
inhibitive effects to recover (\emph{Right}).
}
\label{fig:hawkes-simulated-data}
\end{figure}

% subsection hawkes_processes (end)

% section experiments (end)

\subsection{Heuristic initialization} % (fold)
\label{sub:heuristic_for_a_wise_start}

The default dual initialization in \cite{shalev2014accelerated} ($\alpha^{(0)} = 0_n$) is not a
feasible dual point.
Instead of setting arbitrarily $\alpha^{(0)}$ to $1_n$, we design, from three properties, a
vector $\kappa \in \Dfsm^n$ that is linearly linked to $\alpha^*$ and then rely on
Proposition~\ref{prop:dual-initial-guess-rescale} to find a heuristic starting point $\alpha^{(0)}$
from $\kappa$ for Poisson regression and Hawkes processes.

\paragraph{Property 1: link with $\|x_i\|$.} % (fold)
\label{par:observation_1_link_to_}
Proposition~\ref{prop:dual-optimum-linked-to-norm} relates exactly $\alpha_i^*$ to
the inverse of the norm of $x_i$.
\begin{proposition}
\label{prop:dual-optimum-linked-to-norm}
For Poisson regression and Hawkes processes,
the value of the dual optimum $\alpha_i^*$ is linearly linked to the inverse of the norm of $x_i$.
Namely, if there is $c_i > 0$ such that $\xi_i = c_i x_i$
for any $i \in \{ 1, \dots n \}$, then $\zeta^*$, the solution of the dual problem
\begin{equation*}
\argmax_{\zeta \in (0, +\infty)^n}
      \frac{1}{n} \sum_{i=1}^n y_i + y_i \log \bigg( \frac{\zeta_i}{y_i} \bigg)
          - \lambda g^* \bigg(\frac{1}{\lambda n} \sum_{i=1}^n \zeta_i \xi_i -
          \frac{1}{\lambda} \psi \bigg),
\end{equation*}
satisfies $\zeta^*_i = \alpha_i^* / c_i$ for all $i = 1, \dots n$.
\end{proposition}
This Proposition is proved in Section~\ref{sub:proposition_prop:dual-optimum-linked-to-norm}.
It suggests to consider $\kappa_i \propto 1 / \|x_i\|$ for all $i = 1, \dots n$.

\paragraph{Property 2: link with $y_i$.} % (fold)
\label{par:observation_2_link_to_y_i}

For Poisson regression and Hawkes processes where $f_i(x) = -y_i \log x$, the second
Karush-Kuhn-Tucker Condition~\eqref{eq:general_primal_dual_relations_at_optimum_alpha}
(see Section~\ref{sec:duality} for more details) writes
\begin{equation*}
    \alpha_i^* = \frac{y_i}{{w^*}^\top x_i}
\end{equation*}
for $i  = 1, \dots, n$.
Hence, $\alpha^*$ and $y$ are correlated (a change in $y_i$ only leads to a minor change in $w^*$),
so we will consider $\kappa_i \propto y_i / \|x_i\|$.

% paragraph observation_2_link_to_y_i (end)

\paragraph{Property 3: link with the features matrix.} % (fold)
\label{par:observation_3_link_to_the_features_matrix}

The inner product ${w^*}^\top x_i$ is positive and at the optimum, the Karush-Kuhn-Tucker
Condition~\eqref{eq:general_primal_dual_relations_at_optimum_w}
(which links $w^*$ to $x_i$ through $\alpha_i^*$) tells that $\alpha_i^*$ is likely to be large if
$x_i$ is poorly correlated to other features, i.e. if $x_i^\top \sum_{j=1}^n x_j$ is small.
Finally, the choice
\begin{equation}
\label{eq:kappa-linear-dual-guess}
    \kappa_i = \frac{y_i}{x_i^\top \sum_{j=1}^n x_j}
\end{equation}
for $i = 1, \dots, n$, takes these three properties into account.

Figure~\ref{fig:poisson-unscaled-dual-inits} plots the optimal dual variables $\alpha^*$ from
the Poisson regression experiments of Section~\ref{sub:poisson_regression_expe} against the the
$\kappa$ vector from Equation~\ref{eq:kappa-linear-dual-guess}.
We observe in these experiments a good correlation between the two, but $\kappa$ is only a good
guess for initialization $\alpha^{(0)}$ up to a multiplicative factor that the following
proposition aims to find.
\begin{figure}
\centering
\includegraphics[width=0.75\textwidth]{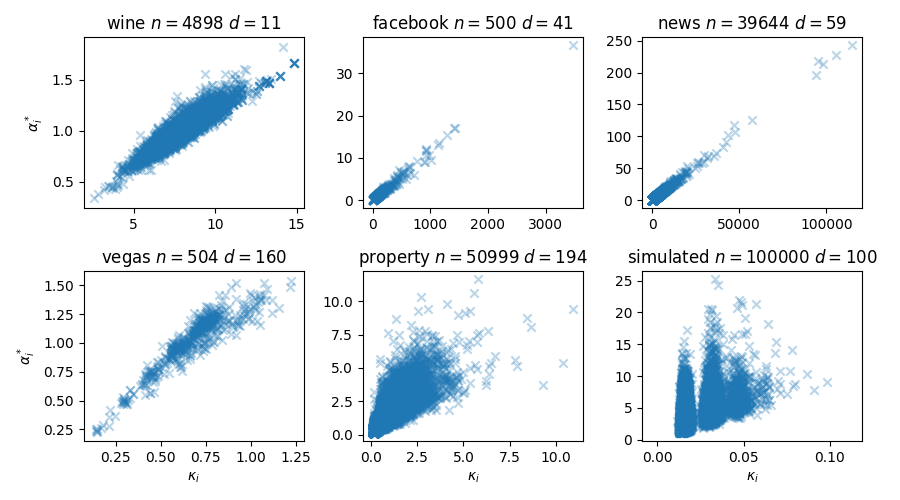}
\caption{Value of $\alpha_i^*$ given $\kappa_i$ for $i=1, \ldots, n$.
There is a linear link relating initial guess $\kappa_i$ to the dual optimum $\alpha_i^*$ on
Poisson datasets but the amplitude is not adjusted yet.
}
\label{fig:poisson-unscaled-dual-inits}
\end{figure}
\begin{proposition}
\label{prop:dual-initial-guess-rescale}
For Poisson regression and Hawkes processes and $g(w) = \tfrac{1}{2} \|w\|^2$, if we constraint the
dual solution ${\alpha^* \in (0, +\infty)^n}$ to be collinear with a given vector
$\kappa \in (0, +\infty)^n$,
i.e. $\alpha^* = \bar{\alpha} \kappa$ for some $\bar{\alpha} \in \R,$ then the optimal value for
$\bar{\alpha}$ is given by
\begin{equation*}
\bar{\alpha} =
    \frac{\psi^\top \chi_\kappa
          + \sqrt{
              (\psi^\top \chi_\kappa)^2
              + 4 \lambda \| \chi_\kappa \|^2  \tfrac{1}{n} \sum_{i=1}^{n} y_i }
          }
          {2  \| \chi_\kappa \|^2}
\quad \text{with} \quad
\chi_\kappa = \frac{1}{n} \sum_{i=1}^{n} \kappa_i x_i.
\end{equation*}
Combined with the previous Properties, we suggest to consider
\begin{equation}
  \label{eq:dual_initialization}
  \alpha^{(0)}_i = \bar{\alpha} \kappa_i
\end{equation}
as an initial point, where $\kappa_i$ is defined in Equation~\eqref{eq:kappa-linear-dual-guess}.
\end{proposition}
This Proposition is proved in Section~\ref{sub:proposition_prop:dual-initial-guess-rescale}.
Figure~\ref{fig:poisson-scaled-dual-inits} presents the values of $\alpha_i^*$ given its initial
value $\alpha_i^{(0)}$ for $i = 1, \ldots, n$ and shows that the rescaling has worked properly.
We validate this heuristic initialization by showing that it leads to a much faster convergence in
Figure~\ref{fig:poisson-dual-inits-convergence} below.
Indeed, we observe that SDCA initialized with Equation~\eqref{eq:dual_initialization} reaches
optimal objective much faster than when initialization consists in setting all dual variables
arbitrarily to 1.
\begin{figure}
\centering
\includegraphics[width=0.75\textwidth]{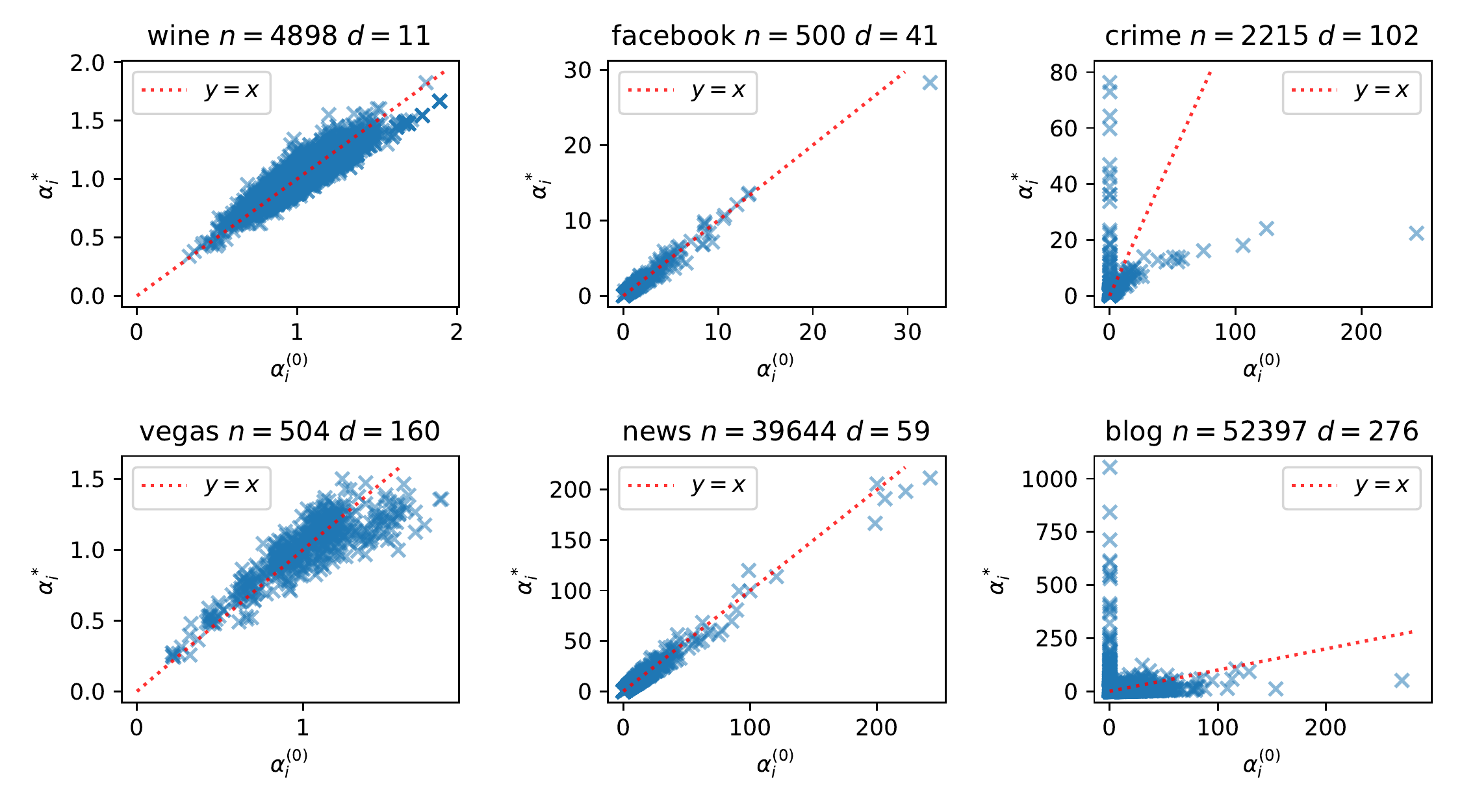}
\caption{
Value of $\alpha_i^*$ given $\alpha_i^{(0)}$ from Equation~\eqref{eq:dual_initialization}
for $i=1, \ldots, n$.
These values are close an correlated which makes $\alpha^{(0)}$ a good initialization value.
}
\label{fig:poisson-scaled-dual-inits}
\end{figure}

\begin{figure}
\centering
\includegraphics[width=0.8\textwidth]{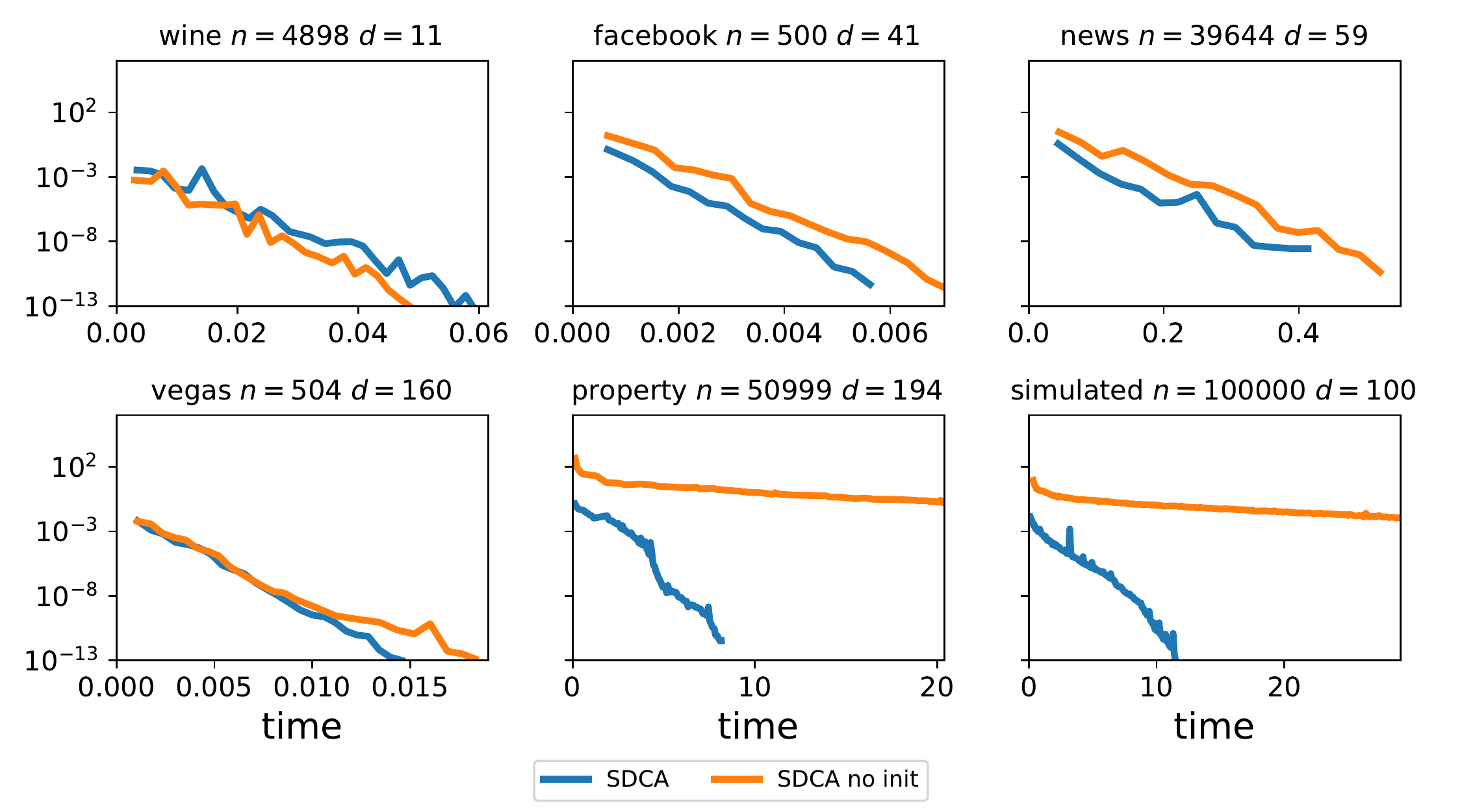}
\caption{Convergence over time of SDCA with wise initialization
from Equation~\eqref{eq:dual_initialization} and SDCA arbitrarily initialized with
$\alpha^{(0)} =1$.
}
\label{fig:poisson-dual-inits-convergence}
\end{figure}

\subsection{Using mini batches} % (fold)
\label{sub:optimizing_over_several_indices}

At each step $t$, SDCA \cite{shalev2013stochastic} maximizes the dual objective by
picking one index $i_t \in \{1, \dots, n\}$ and maximizing the dual objective over the coordinate
$i_t$ of the dual vector $\alpha$, and sets
\begin{equation*}
\alpha_i^{t+1} = \argmax_{v \in \Dfsm}
     D(\alpha_1^t, \dots , \alpha_{i_t -1}^t, v, \alpha_{i_t + 1}^t, \dots, \alpha_n).
\end{equation*}
In some cases this maximization has a closed-form solution, such as for Poisson regression (see
Proposition~\ref{prop:closed_form_solution_log_losses}) or least-squares regression where
$f_i(w) = (y_i - w^\top x_i)^2$ leads to the explicit update
\begin{equation*}
    \alpha_i^{t+1} =
    \alpha_i^t + \frac{y_i - w^\top x_i - \alpha_i^{t}}{1 + (\lambda n)^{-1} \|x_i\|^2}.
\end{equation*}
In some other cases, such as logistic regression, this closed form solution cannot be exhibited and
we must perform a Newton descent algorithm.
Each Newton step consists in computing $\partial D(\alpha) / \partial \alpha_i$ and
$\partial^2 D(\alpha) / \partial \alpha_i^2$, which are one dimensional operations given
$\|x_i\|^2$ and $w^\top x_i$.
Hence, in a large dimensional setting, when the observations $x_i$ have many non zero entries,
the main cost of the steps resides mostly in computing $\|x_i\|^2$ and $w^\top x_i$.
Since $\|x_i\|^2$ and $w^\top x_i$ must also be computed when using a closed-form solution, using
Newton steps instead of the closed-form is eventually not much more computationally expensive.
So, in order to obtain a better trade-off between Newton steps and inner-products computations, we
can consider more than a single index on which we maximize the dual objective.
This is called the mini-batch approach, see Stochastic Dual Newton Ascent (SDNA) \cite{qu2016sdna}.
It consists in selecting a set $\mathcal{I} \subset \{ 1, \dots, n\}$ of $p$ indices at each
iteration $t$.
The value of $\alpha_i^{t+1}$ becomes in this case
\begin{equation*}
\alpha_i^{t+1} = \argmax_{v \in (\Dfsm)^p}
    D(b_1, \dots, b_n)
    \quad \text{where} \quad
    b_i = \begin{cases}
      v_j & \text{if} \quad i \in \mathcal{I} \text{ and } j \text{ is the position of } i
            \text{ in } \mathcal{I}\\
      \alpha_i^t & \text{otherwise.}
    \end{cases}
\end{equation*}
The two extreme cases are $p = 1$, which is the standard SDCA algorithm, and $p = n$ for which we
perform a full Newton algorithm.
After computing the inner products $w^\top x_i$ and $x_i^\top x_j$ for all
$(i, j) \in \mathcal{I}^2$ each iteration will simply performs up to 10 Newton steps in which the
bottleneck is to solve a $p \times p$ linear system.
This allows to better exploit curvature and obtain better
convergence guarantees for gradient-Lipschitz losses \cite{qu2016sdna}.

We can apply this to Poisson regression and Hawkes processes where $f_i(x) = -y_i \log x$.
The maximization steps of Line~\ref{lst:line:local_max} in Algorithm~\ref{alg:prox-shifted-sdca}
is now performed on a set of coordinates $\mathcal{I} \subset \{ 1, \dots n \}$ and consists in
finding
\begin{equation*}
  \max_{\alpha_i ; i \in \mathcal{I}} D_\mathcal{I}^t(\alpha_\mathcal{I}) \quad \text{where} \quad
  D_\mathcal{I}^t(\alpha_\mathcal{I}) =
      \frac{1}{n} \sum_{i \in \mathcal{I}} \Big( y_i + y_i \log \tfrac{\alpha_i}{y_i} \Big)
       - \frac{\lambda}{2} \Big\|
               w^t + \frac{1}{\lambda n} \sum_{i \in \mathcal{I}} (\alpha_i - \alpha_i^t ) x_i
            \Big\|^2,
\end{equation*}
where we denote by $\alpha_\mathcal{I}$ the sub-vector of $\alpha$ of
size $p$ containing the values of all indices in $\mathcal{I}$.
We initialize the vector $\alpha_\mathcal{I}^{(0)} \in (\Dfsm)^p$ to the corresponding values of
the coordinates of $\alpha^t$ in $\mathcal{I}$ and then perform the Newton steps, i.e.
\begin{equation}
  \label{eq:newton-linear-system}
  \alpha_\mathcal{I}^{k+1} = \alpha_\mathcal{I}^{k} - \Delta \alpha_\mathcal{I}^k
  \quad \text{where } \Delta \alpha_I^k \text{ is the solution of} \quad
  \nabla^2 D_\mathcal{I}^t(\alpha_\mathcal{I}^k) \Delta \alpha_\mathcal{I}^k
    = \nabla D_\mathcal{I}^t(\alpha_\mathcal{I}^k).
\end{equation}
The gradient $\nabla D_\mathcal{I}^t(\alpha_\mathcal{I}^k)$
and the hessian $\nabla^2 D_\mathcal{I}^t(\alpha_\mathcal{I}^k)$ have the following explicit
formulas:
\begin{equation*}
(\nabla D_\mathcal{I}^t(\alpha_\mathcal{I}^k))_i =
    \frac{\partial D(\alpha_\mathcal{I}^k)}{\partial \alpha_i} =
    \frac{1}{n} \Big(
        \frac{y_i}{\alpha_i^k}
        - {w^t}^\top x_i
        - \frac{1}{\lambda n} \sum_{j \in \mathcal{I}} (\alpha_j^k - \alpha_j^t ) x_j^\top x_i
      \Big),
\end{equation*}
and
\begin{equation*}
\big( \nabla^2 D_I^t(\alpha_I^k) \big)_{i, j}
   = \frac{\partial^2 D(\alpha_I^k)}{\partial \alpha_i \partial \alpha_j} =
   - \frac{1}{n} \Big( \frac{y_i}{{\alpha_i^k}^2} \mathbf{1}_{i = j}
                + \frac{1}{\lambda n} x_i^\top x_j \Big).
\end{equation*}
Note that $D_I^t$ is a concave function hence $- \nabla^2 D_I^t(\alpha_I^k)$ will be positive
semi-definite and the system in Equation~\eqref{eq:newton-linear-system} can be solved very
efficiently with BLAS and LAPACK libraries.
Let us explicit computations when $p=2$.
Suppose that $\mathcal{I} = \{i, j\}$ and put
\begin{equation*}
  \delta_i = \alpha_i - \alpha_i^{(t-1)}, \quad p_i = x_i^\top w^{(t- 1)}
  \quad \text{and} \quad g_{ij} = \frac{x_i^\top x_j}{\lambda n}.
\end{equation*}
The gradient and the Hessian inverse are then given by
\begin{equation*}
\nabla D(\alpha_I) = \frac{1}{n}
\begin{bmatrix}
    \frac{y_i}{\alpha_i} - p_i - \delta_i g_{ii} - \delta_j g_{ij} \\
    \frac{y_j}{\alpha_j} - p_j - \delta_j g_{jj} - \delta_i g_{ij}
 \end{bmatrix},
\end{equation*}
and
\begin{equation*}
 \nabla^2 D(\alpha_I)^{-1} =
     \frac{n^2}{(\frac{y_i}{\alpha_i^2} + g_{ii})(\frac{y_i}{\alpha_i^2} + g_{ii}) - g_{ij}^2}
\begin{bmatrix}
    - \frac{y_j}{\alpha_j^2} - g_{jj} & g_{ij} \\
    g_{ij} & - \frac{y_i}{\alpha_i^2} - g_{ii}
 \end{bmatrix}.
\end{equation*}
This direct computation leads to even faster computations than using the dedicated libraries.
We plot in Figure~\ref{fig:poisson-batches-convergence} the convergence speed for three sizes of
batches 1, 2 and 10.
Note that in all cases using a batch of size $p=2$ is faster than standard SDCA.
Also, in the last simulated experiment where $d$ has been set on purpose to $1000$, the solver
using batches of size $p=10$ is the fastest one.
The bigger number of features $d$ gets, the better are solvers using big batches.

\begin{figure}
\centering
\includegraphics[width=0.85\textwidth]{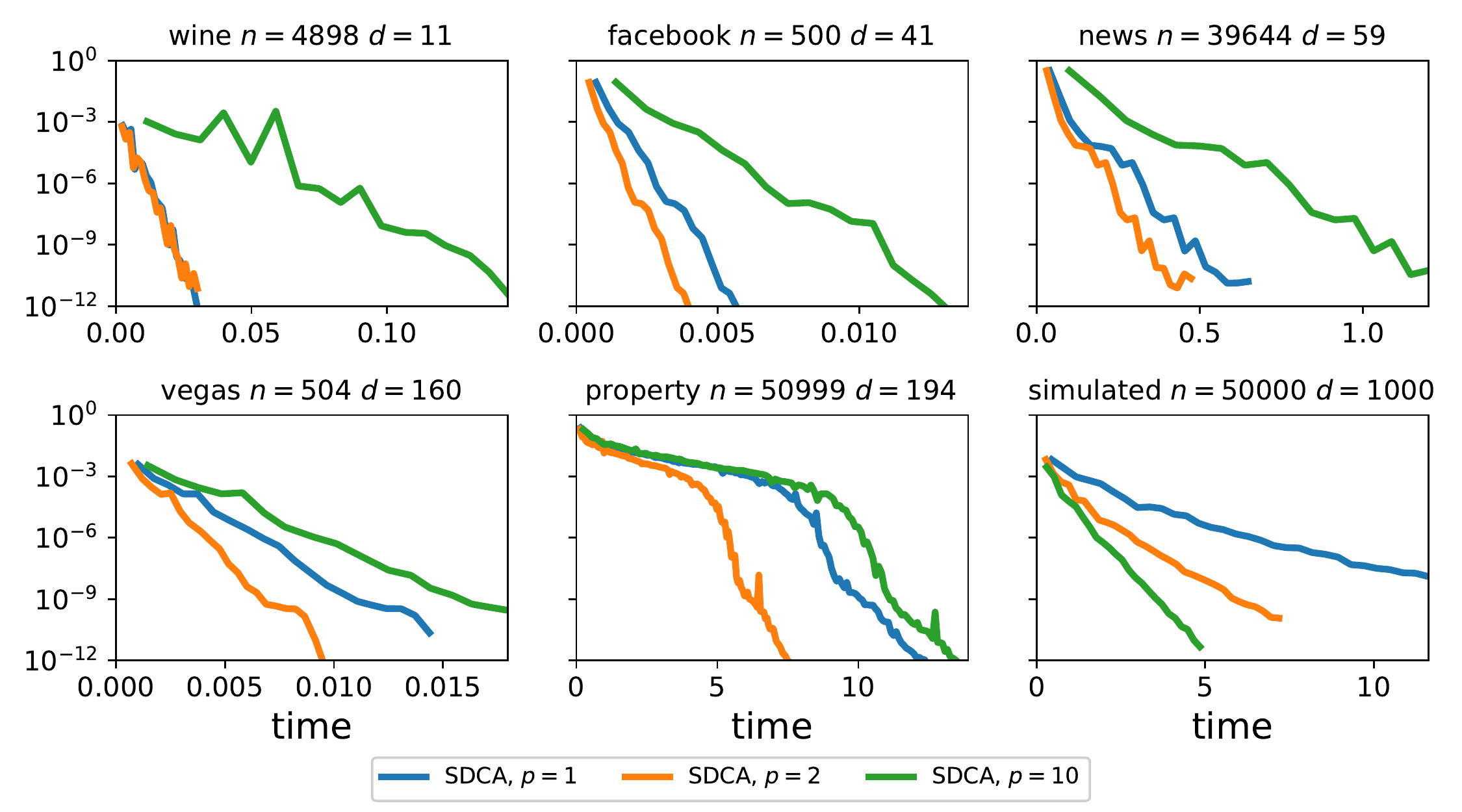}
\caption{Convergence speed comparison when the number of indices optimized at each step changes.}
\label{fig:poisson-batches-convergence}
\end{figure}

% subsection optimizing_over_several_indices (end)

\subsection{About the pessimistic upper bounds} % (fold)
\label{sub:about_the_pessimistic_upper_bounds}

The generic upper bounds derived in Proposition~\ref{prop:dual_large_bound} are general but
pessimistic as they depend linearly on $n$.
In fact this dependence is also observed in the NoLips algorithm \cite{bauschke2016descent} where
the rate depends on a constant $L = \sum_{i}^n y_i$.
Note that, for Nolips algorithms, $L$ is involved in the step size definition and leads to step too
small to be used in practice but that in our algorithm, these bounds have little or even no impact
in practice (see Remark~\ref{rmk:closed-form-bounded}) and are mainly necessary for convergence
guarantees.
These bounds are derived by lower bounding $x_i^\top \sum_{j \neq i} \alpha_j^* x_j$ by $0$.
This lower bound is very conservative and can probably be tightened by setting specific hypotheses
on the dataset, for example on the Gram matrix
($[G]_{i,j} = x_i^\top x_j$ for $i,j = 1, \ldots, n$).
For Poisson regression, this lower bound is reached in the extreme case where all observations are
orthogonal (all entries of $G$ are zero except on the diagonal).
Then $\psi^\top x_i = \frac{1}{n} \|x_i\|^2$ and the upper bounds from
Proposition~\ref{prop:dual_large_bound} become
\begin{equation*}
  \beta_i = \frac{1}{2} + \sqrt{\frac{1}{4} +
  \frac{\lambda n y_i}{ \|x_i\|^2}}
\end{equation*}
for $i= 1, \ldots, n$.
In this extreme case, the bounds are $O(\sqrt n)$ instead of $O(n)$ as stated in
Proposition~\ref{prop:dual_large_bound}.
Experimentally, we do not observe a dependence in $O(n)$ either.
Figure~\ref{fig:poisson-duals-n-evolution} shows the evolution of the maximum optimal dual
obtained ($\max_{i=1, \ldots, n} \alpha_i^*$) for the six datasets considered in
Section~\ref{sub:poisson_regression_expe} for Poisson regression, on an increasing fraction
of the dataset.
These values are averaged over 20 samples and we provide the associated 95\% confidence interval
on this value.
We observe that $\max_{i=1, \ldots, n} \alpha_i^*$ has a much lower dependence in $n$ than the
bounds given by Proposition~\ref{prop:dual_large_bound}.
\begin{figure}
\centering
\includegraphics[width=0.85\textwidth]{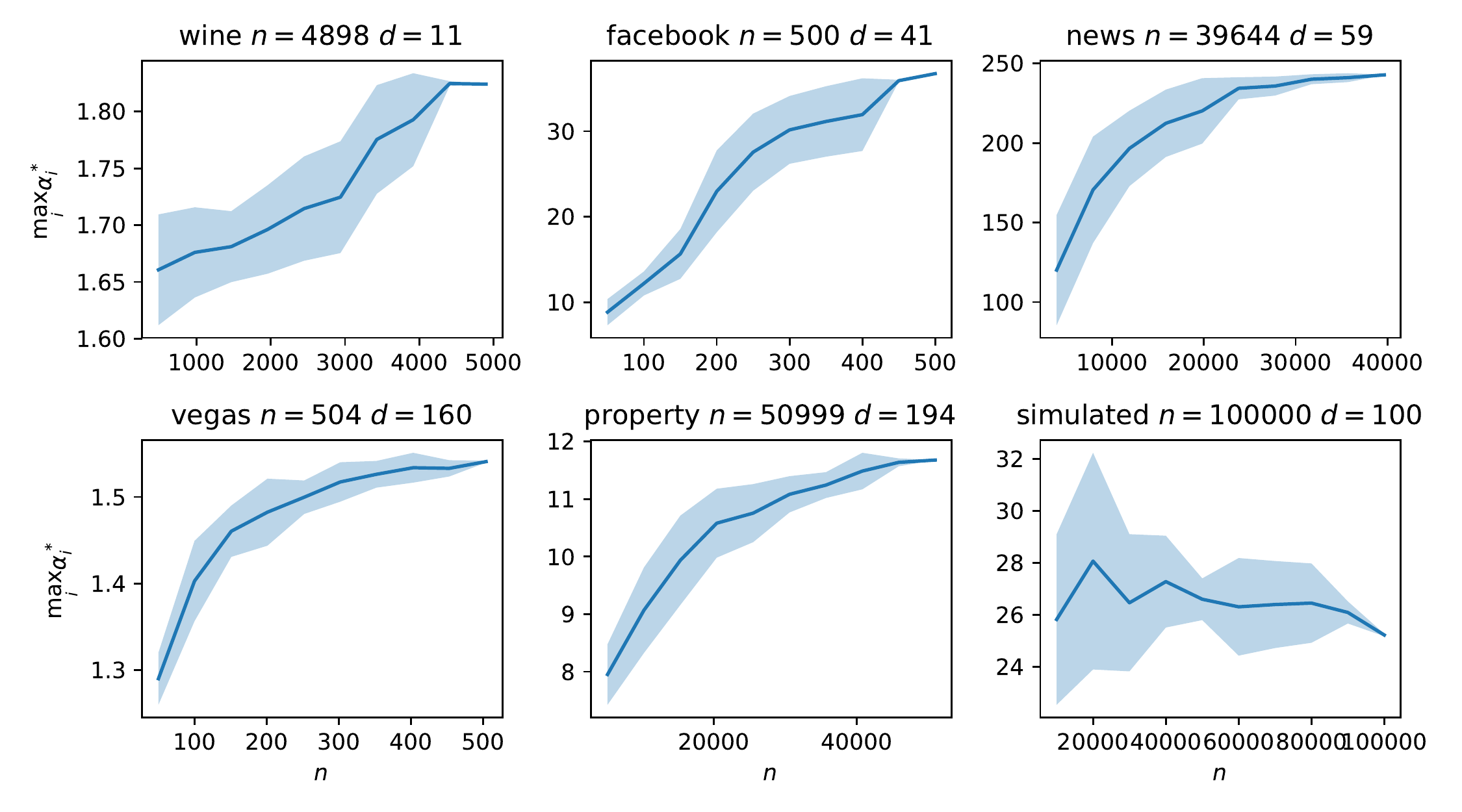}
\caption{Evolution of $\max_{i=1, \ldots, n} \alpha_i^*$ for an increasing value of $n$.
We observe that $\max_{i=1, \ldots, n} \alpha_i^*$ is not increasing linearly with $n$ as quick as
the bound $\beta_i$ obtained in Proposition~\ref{prop:dual_large_bound}.}
\label{fig:poisson-duals-n-evolution}
\end{figure}

% subsection about_the_pessimistic_upper_bounds (end)

\section{Conclusion} % (fold)
\label{sec:conclusion}

This work introduces the $\log$-smoothness assumption in order to derive improved linear rates for
SDCA, for objectives that do not meet the gradient-Lipschitz assumption.
This provides, to the best of our knowledge, the first linear rates for a stochastic first order
algorithm without the gradient-Lipschitz assumption.
The experimental results also prove the efficiency of SDCA to solve such problems and its ability
to deal with the open polytope constraints, improving the state-of-the-art.
Finally, this work also presents several variants of SDCA and experimental heuristics to make the
most of it on real world datasets.
Future work could provide better linear rates under more specialized assumptions on the Gram matrix,
as observed on numerical experiments.
Also, to extend this work to more applications, we aim to find a generalization of the $\log$
smoothness assumption such as what \cite{sun2017generalized} has done for self-concordance.

\subsection*{Acknowledgments}
We would like to acknowledge support for this project from the Datascience Initiative
of \'Ecole polytechnique

% \FloatBarrier
% \clearpage
% \newpage

% {\noindent \huge \textbf{Appendices}}

\section{Proofs} % (fold)
% \label{sec:some_technical_details}
\label{sec:proofs-chap-dual}
% section some_technical_details (end)

We start this Section by providing extra details on the derivation of the dual problem and the
proximal version of SDCA.
We then provides the proofs of all the results stated in the chapter, namely
Proposition~\ref{prop:log-smooth-second-order},
Proposition~\ref{prop:hyp_dual_hessian_equivalence},
Theorem~\ref{th:general-convergence},
Theorem~\ref{th:convergence-is},
Proposition~\ref{prop:closed_form_solution_log_losses},
Proposition~\ref{prop:dual_large_bound},
Remark~\ref{rmk:closed-form-bounded},
Proposition~\ref{prop:dual-optimum-linked-to-norm}
and Proposition~\ref{prop:dual-initial-guess-rescale}.

\subsection{Duality and proof of Proposition~\ref{prop:strong-duality}} % (fold)
\label{sec:duality}

It is not straightforward to obtain strong duality for a convex problem with strict inequalities
(such as the ones enforced by the polytope $\Pi(X)$ from Equation~\ref{eq:feasible_polytope}).
To bypass this difficulty we consider the same problem but constrained on a closed set and show how
it relates to the original Problem~\eqref{eq:primal}.
But first we formulate the two following lemmas.
\begin{lemma}
\label{lemma:equivalence-constrained-primal}
It exists $\varepsilon > 0$ such that the following problem
\begin{equation*}
  \min_{w \in \Pi_{|\varepsilon}(X)} P(w)
  \quad \text{where} \quad
  \Pi_{|\varepsilon}(X) = \{ w \in \R^d \; : \; \forall i \in \{1 , \dots, n\}, \;
                        w^\top x_i \geq \varepsilon \}
\end{equation*}
has a solution $w^*$ that is also the solution of the original Problem~\eqref{eq:primal}.
\end{lemma}
\begin{proof}
First, notice that this problem has a unique solution which is unique as we minimize a convex
function on a closed convex set $\Pi(X)$.
Also, as the function $w \mapsto \psi^\top w + \lambda g(w)$ is strongly convex, since
$g$ is strongly convex, it is lower bounded.
We denote by $M \in \R$ a lower bound of this function.
Then, we consider $w^0 \in \Pi(X)$, and choose $\varepsilon$ sufficiently small such that for
all $i = 1, \dots, n$,
\begin{equation*}
\forall t < \varepsilon, \; f_i(t) > n P (w^0) - nM,
\end{equation*}
this value of $\varepsilon$ always exists since by assumption
$\lim_{t \rightarrow 0} f_i(t) = + \infty$ for all $i=1, \dots, n$.
For any $w_\varepsilon \in \Pi(X) \setminus \Pi_\varepsilon(X)$ (so a $w_\varepsilon$ is such
that $\exists i \in \{1 , \dots, n\}, \; w_\varepsilon^\top x_i < \varepsilon$), we thus have
\begin{equation*}
P(w_\varepsilon) > \psi^\top w_\varepsilon + P(w^0) - M + \lambda g(w_\varepsilon) \geq P(w^0).
\end{equation*}
Hence, for such a value of $\varepsilon$, the solution to the original Problem~\eqref{eq:primal}
is necessarily in $\Pi_\varepsilon(X)$ and both the problems constrained on $\Pi_\varepsilon(X)$ and
$\Pi(X)$ share the same solution $w^*$.
\end{proof}
\begin{lemma}
\label{lemma:fenchel_conjugates_equality}
For all $i = 1, \dots, n$, if we define by
\begin{equation*}
  \forall \alpha_i \in \Dfs, \;
  f^*_{i|\varepsilon}(v) := \max_{u \geq \varepsilon} uv - f_i(u),
\end{equation*}
then $f^*_{i|\varepsilon}$ is equal to the Fenchel conjugate of $f_i$, $f_i^*$, on
$\{ v \; ; \; \exists u \geq \varepsilon \; ; \; v = f_i'(u) \}$.
\end{lemma}
\begin{proof}
For all $i = 1, \dots, n$, the functions $f_i$ are convex and differentiable.
Hence, by Fermat's rule if $\exists u^* \geq \varepsilon; \; v = f_i'(u^*)$ then
$f^*_{i|\varepsilon} (v) = \max_{u \geq \varepsilon} uv - f_i(u) = u^* f_i'(u^*) - f_i(u^*)$.
Likewise, the maximization step in the computation of $f_i^*$ would share the same maximizer and
${f_i^*} (v) = u^* f_i'(u^*) - f_i(u^*)$ as well.
Hence,
\begin{equation*}
\forall v \; \text{such that} \; \exists u \geq \varepsilon \; ; \; v = f_i'(u); \;
f^*_{i|\varepsilon} (v) = {f_i^*} (v).
\end{equation*}
\end{proof}
We form the dual of the problem constrained on $\Pi_\varepsilon(X)$ where $\varepsilon$ is such
that Lemma~\ref{lemma:equivalence-constrained-primal} applies.
We replace the inner products $w^\top x_i$ by the scalars $u_i$ for $i = 1, \ldots, n$ and
their equality is constrained to form the strictly equivalent problem:
\begin{equation*}
\min_{\substack{w \in \R^d, u \in {[\varepsilon, +\infty)}^n\\
      \forall i, u_i = w^\top x_i}}
        \psi^\top w + \frac{1}{n} \sum_{i=1}^n f_i(u_i) + \lambda g(w).
\end{equation*}
We maximize the Lagrangian to include the constraints.
This introduces the vector of dual variables $\alpha \in \R^n$ as following:
\begin{equation*}
\max_{\alpha \in \R^n} \min_{\substack{w \in \R^d \\ u \in {[\varepsilon, +\infty)}^n}}
\psi^\top w + \frac{1}{n} \sum_{i=1}^n f_i(u_i) + \lambda g(w) +
    \frac{1}{n} \sum_{i=1}^n \alpha_i (u_i - x_i^\top w)
\end{equation*}
that leads to the corresponding dual problem:
\begin{equation*}
\max_{\alpha \in (\Dfsme)^n} D_{|\varepsilon} (\alpha), \quad
    D_{|\varepsilon}(\alpha) =
    \frac{1}{n} \sum_{i=1}^n - f^*_{i|\varepsilon} (-\alpha_i)
    - \lambda g^* \bigg(\frac{1}{\lambda n} \sum_{i=1}^n \alpha_i x_i
                        - \frac{1}{\lambda} \psi \bigg),
\end{equation*}
where $\Dfsme$ is the domain of all $f^*_{i|\varepsilon}$.
The primal problem constrained on $\Pi_\varepsilon(X)$ verifies the Slater's conditions so strong
duality holds and the maximizer of $D_{|\varepsilon}$, $\alpha^*_{|\varepsilon}$, is reached.
As $D_{|\varepsilon}$ is concave, $\alpha^*_{|\varepsilon}$ is the only vector such that
$\nabla D_{|\varepsilon} (\alpha^*_{|\varepsilon}) = 0$.
Also, as strong duality holds, we can relate $\alpha^*_{|\varepsilon}$ to the primal optimum
though the Karush-Kuhn-Tucker condition
\begin{equation*}
\alpha^*_{i|\varepsilon} = - {f^*_{i|\varepsilon}}'({w^*}^\top x_i),
\end{equation*}
where $w^*$ is such that ${w^*}^\top x_i \geq \varepsilon$
(see Lemma~\ref{lemma:equivalence-constrained-primal}).
Hence Lemma~\ref{lemma:fenchel_conjugates_equality} applies and the dual formulation of the original
Problem~\eqref{eq:primal} that writes
\begin{equation*}
\max_{\alpha \in (\Dfsm)^n} D (\alpha), \quad
    D(\alpha) =
    \frac{1}{n} \sum_{i=1}^n - f_i^* (-\alpha_i)
    - \lambda g^* \bigg(\frac{1}{\lambda n} \sum_{i=1}^n \alpha_i x_i
                        - \frac{1}{\lambda} \psi \bigg),
\end{equation*}
is such that
$\nabla D(\alpha^*_{|\varepsilon}) = \nabla D_{|\varepsilon} (\alpha^*_{|\varepsilon}) = 0$.
Since $D(\alpha)$ is concave, this means that $\alpha^*_{|\varepsilon} = \alpha^*$ where $\alpha^*$
is the solution of the dual formulation of the original Problem~\eqref{eq:primal}.
Thus, the Karush-Kuhn-Tucker conditions that link the primal and dual optima of the problem
constrained on $\Pi_{|\varepsilon}(X)$ also links the primal and dual optima of the
original Problem~\eqref{eq:primal}.
The first one is given in Equation~\eqref{eq:general_primal_dual_relations_at_optimum_w} and the
second one writes
\begin{equation}
\label{eq:general_primal_dual_relations_at_optimum_alpha}
  \alpha_i^* = - {f_i^*}'({w^*}^\top x_i)
\end{equation}
for any $i \in \{1, \dots n\}$.
From the first we can define two functions linking vector $w \in \R^d$ to
$\alpha \in (\Dfsm)^{n}$ and such that $w(\alpha^*) = w^*$ and
\begin{equation}
  \label{eq:primal_dual_relation}
  v(\alpha) = \frac{1}{\lambda n} \sum_{i=1}^n \alpha_i x_i -\frac{1}{\lambda} \psi
  \quad \text{and} \quad
  w(\alpha) = \nabla g^*\big( v(\alpha) \big).
\end{equation}

% section duality (end)

\subsection{Proximal algorithm} % (fold)
\label{sec:proximal_algorithm_supp}

Given that $g^*$ is smooth since its Fenchel conjugate is strongly convex, the gradient-Lipschitz
property from Definition~\ref{def:smoothness} below entails
$g^*(v + \Delta v) \leq g^*(v) + \nabla g^*(v)^\top \Delta v + \tfrac{1}{2} \|\Delta v\|^2$.
Hence, maximization step of Algorithm~\ref{alg:general-shifted-sdca}, namely,
\begin{equation*}
\argmax_{\alpha_i \in \Dfsm}
      - f^*_i (-\alpha_i)
          - \lambda n g^* \big(v^{(t-1)} + (\lambda n)^{-1}(\alpha_i - \alpha_i^{(t-1)}) x_i
          \big),
\end{equation*}
where
$v^{(t - 1)} = \frac{1}{\lambda n} \sum_{i=1}^n \alpha_i^{(t-1)} x_i -\frac{1}{\lambda} \psi $
can be simplified by setting $\alpha_i^{t}$ such that it maximizes the lower bound
\begin{equation}
\label{eq:modified-argmax}
 \alpha_i^{t} = \argmax_{\alpha_i \in \Dfsm}
          - f^*_i (-\alpha_i)
          - \lambda n \bigg(
              g^* \big(v^{(t-1)}\big)
              + \frac{\alpha_i - \alpha_i^{(t-1)}}{\lambda n} x_i^\top \nabla g^* (v^{(t-1)})
              + \frac{1}{2} \bigg(
                  \frac{\alpha_i - \alpha_i^{(t-1)}}{\lambda n} \bigg)^2
                  \| x_i \|^2
                \bigg).
\end{equation}
Setting $w^{(t-1)} = \nabla g^* (v^{(t-1)})$ and discarding constants terms leads to the equivalent
relation,
\begin{equation*}
\alpha_i^{t} = \argmax_{\alpha_i \in \Dfsm}
      - f^*_i (-\alpha_i)
          - \frac{\lambda n}{2} \Big\| w^{(t-1)}
          - (\lambda n)^{-1} (\alpha_i - \alpha_i^{(t-1)}) x_i  \Big\|^2.
\end{equation*}
While convergence speed is guaranteed for any 1-strongly convex $g$, to simplify the algorithm
we will consider that $g$ is not only 1-strongly convex but that that it can also be decomposed as
\begin{equation*}
  g(w) = \tfrac{1}{2} \|w\|^2 + h(w)
\end{equation*}
where $h$ is a prox capable function.
With Proposition~\ref{prop:convex-conjugate-gradient} below the relation between $w^{t}$ and $v^t$
becomes
\begin{equation*}
  w^t
    = \nabla g^*(v^t)
    = \argsup_{u \in \R^d} \Big( u^\top v^t - \tfrac{1}{2} \|u\|^2 - h(u) \Big)
    = \arginf_{u \in \R^d} \Big(\tfrac{1}{2} \|v^t - u\|^2 + h(u) \Big),
\end{equation*}
which is the proximal operator stated in Definition~\ref{def:proximal-operator} below:
$w^t = \prox_{h} (v(\alpha^t))$.

\subsection{Preliminaries for the proofs} % (fold)
\label{sub:preliminaries}

Let us first recall some definitions and basic properties.

\begin{definition}
  \emph{Strong convexity.}
  A differentiable convex function $f : \Df \rightarrow \R$ is $\gamma$-strongly convex if
  \begin{equation}
    \label{eq:strong-convex-zero-order-lower-bound}
    \forall x, y \in \Df, \quad
    f(y) \geq f(x) + f' (x)^\top (y - x) + \frac{\gamma}{2} \|y - x\|^2.
  \end{equation}
  This is equivalent to
  \begin{equation}
    \label{eq:strong-convex-first-order-lower-bound}
    \forall x, y \in \Df, \quad
    (f'(y) - f'(x)) (y - x) \geq \gamma \|y - x\|^2.
  \end{equation}
\end{definition}

\begin{definition}
  \label{def:smoothness}
  \emph{Smoothness.}
  A differentiable convex function $f : \Df \rightarrow \R$ is $L$-smooth or $L$-gradient Lipschitz
  if
  \begin{equation*}
    \forall x, y \in \Df, \quad
    f(y) \leq f(x) + f'(x) (y - x) + \frac{L}{2} \|y - x\|^2.
  \end{equation*}
  This is equivalent to
  \begin{equation*}
    \forall x, y \in \Df, \quad
    (f'(y) - f'(x)) (y - x) \leq L \|y - x\|^2.
  \end{equation*}
\end{definition}

\begin{definition}
  \label{def:convex-conjugate}
  \emph{Fenchel conjugate.}
  For a convex function $f : \Df \rightarrow \R$ we call Fenchel conjugate the function $f^*$
  defined by
  \begin{equation}
  \label{eq:def-convex-conjugate}
    f^* : \Dfs \rightarrow \R, \quad \text{st.} \quad
    f^*(v) = \sup_{u \in \Df} \big( u^\top v - f(u) \big).
  \end{equation}
\end{definition}

\begin{proposition}
  \label{prop:convex-conjugate-gradient}
  For a convex differentiable function $f$, the gradient of its differentiable Fenchel conjugate
  $f^*$ is the
  maximizing argument of~\eqref{eq:def-convex-conjugate}\textup:
  \begin{equation*}
    {f^*}' (v) = \argsup_{u \in \Df} \big( u^\top v - f(u) \big).
  \end{equation*}
\end{proposition}

\begin{proposition}
  \label{prop:convex-conjugate-reciprocity}
  For a convex differentiable function $f$ and its differentiable Fenchel conjugate $f^*$ we have
  \begin{equation*}
    \forall u \in \Df, {f^*}'(f'(u)) = u
    \quad \text{and} \quad
    \forall v \in \Dfs, f'({f^*}'(v)) = v.
  \end{equation*}
  This leads to
  \begin{equation*}
  \forall u \in \Df, \forall v \in \Dfs, \quad
  f'(u) = v \Leftrightarrow u = {f^*}'(v).
  \end{equation*}
\end{proposition}
Note that if $f$ is $\gamma$-strongly convex (respectively $L$-smooth), then its Fenchel conjugate
$f^*$ is $1 /\gamma$ smooth (respectively $1 / L$ strongly convex).
We also recall results from \cite{nesterov1994interior} on self-concordant functions introduced in
Definition~\ref{def:self-concordance}.
This concept is widely used to study losses involving logarithms.
For the sake of clarity, the results will be presented for functions whose domain $\Df$ is a
subset of $\R$ as this leads to lighter notations.
Unlike smoothness and strong convexity, this property is affine invariant.
From this definition, some inequalities are derived in \cite{nesterov1994interior}.
Two of them provide lower bounds that are comparable to strong convexity inequalities:
\begin{equation}
  \label{eq:self-concordant-zero-order-lower-bound}
  \forall x, y \in \Df, \quad
  f(y) \geq f(x) + f' (x) (y - x) + \omega \big( \sqrt{f''(x)} |y - x|\big)
\end{equation}
where $\omega(t) = t - \log (1 + t)$, and
\begin{equation}
    \label{eq:self-concordant-first-order-lower-bound}
  \forall x, y \in \Df, \quad
  (f'(y) - f'(x))^\top (y - x) \geq \frac{f''(x) (y - x)^2}{1 + \sqrt{f''(x)} |y - x|}.
\end{equation}
Finally, we define the proximal operator used to apply the penalization $g$.
\begin{definition}
  \label{def:proximal-operator}
  \emph{Proximal operator.}
  For a convex function $g : \Dg \rightarrow \R$, the proximal operator associated to $g$ is given
  by
  \begin{equation*}
  \prox_g (y) = \argmin_{x \in \Dg} \Big( \frac{1}{2} \|y - x\|^2 + g(x) \Big).
  \end{equation*}
\end{definition}
The proximal operator always exists and is uniquely defined as the minimizer of a strongly convex
function.
Before the proof of Theorem~\ref{th:general-convergence}, we need to introduce new convex
inequalities for $L$-$\log$ smooth functions.
This class of function includes $x \mapsto - L \log x$ which is our function of interest in the
Poisson and in the Hawkes cases.
% These inequalities are built on the model of self-concordant inequalities from
% \cite{nesterov1994interior}.

% section proof_of_theorem_th:convergence_convergence_proof (end)

\subsection{Proof of Proposition~\ref{prop:log-smooth-second-order}} % (fold)
\label{sub:proof_of_proposition_prop:log-smooth-second-order}

\paragraph{First order implies second order}

We start by showing that if $f$ is a L $\log$-smooth function then we can bound its second
derivative by the square of its gradient.
For any $x \in \Df$, we set $y = x + h$ in the Definition~\ref{hyp:new-gradient-assumption}, which
now writes
\begin{equation*}
    \forall x \in \Df, \;
    \forall h \text{ s.t. } (x + h) \in \Df, \;
    \Big| \frac{f'(x) - f'(x + h)}{h} \Big| \leq \frac{1}{L} f'(x) f'(x + h).
\end{equation*}
Taking the limit of the previous inequality when $h$ tends towards $0$ leads to the desired
inequality,
\begin{equation*}
    \forall x \in \Df, \;
    \big| f''(x) \big| \leq \frac{1}{L} f'(x)^2.
\end{equation*}

\paragraph{Second order implies first order}

We now prove that if $f$ is convex strictly monotone, twice differentiable and
$|f''(x)| \leq \frac{1}{L} f'(x)^2$ then $f$
is $L$-$\log$ smooth.
If for all $x \in \Df$, we denote by $\phi : x \mapsto \frac{1}{f'(x)}$, (note that
$\forall x \in \Df, f'(x) \neq 0$ as $f$ is strictly monotone),
then
\begin{equation*}
\forall x \in \Df, \; | \phi'(x) | = \Big| \frac{f''(x)}{f'(x)^2} \Big| \leq \frac{1}{L}.
\end{equation*}
From this inequality, we limit the increasings of the function $\phi$,
\begin{equation*}
\forall x, y \in \Df, \;
- \tfrac{1}{L} |y - x| \leq \phi(y) - \phi(x) \leq \tfrac{1}{L}|y - x|,
\end{equation*}
which rewrites
\begin{equation*}
\forall x, y \in \Df, \;
|\phi(y) - \phi(x) | \leq \tfrac{1}{L} |y - x|
\; \Leftrightarrow \;
\Big|\frac{f'(x) - f'(y)}{f'(x) f'(y)} \Big| \leq \frac{1}{L} |x - y|,
\end{equation*}
that is the definition of a $L$-$\log$ smooth function for a convex strictly monotone function.

% subsection proof_of_proposition_prop:log-smooth-second-order (end)

\subsection{Proof of Proposition~\ref{prop:hyp_dual_hessian_equivalence}} % (fold)
\label{sec:proof_of_proposition_}

We working by exhibiting several statements equivalent to $\log$ smoothness.
First, we divide both sides of the $\log$ smoothness definition by $f'(x)f'(y) > 0$ since $f$ is
strictly monotone,
% \begin{equation*}
%   \Leftrightarrow \quad
%   \forall x, y \in \Dfs,\quad
%    | f'(x) - f'(y) | \leq \frac{1}{L} f'(x) f'(y) | x - y |
% \end{equation*}
\begin{equation*}
  \text{$f$ is $L$-$\log$ smooth}
  \quad \Leftrightarrow \quad
  \forall x, y \in \Df, \;
   \bigg| \frac{1}{f'(y)} - \frac{1}{f'(x)} \bigg| \leq \frac{1}{L} |x - y|.
\end{equation*}
Then we rewrite the equation in the dual space using
Proposition~\ref{prop:convex-conjugate-reciprocity},
\begin{equation*}
  \text{$f$ is $L$-$\log$ smooth}
  \quad \Leftrightarrow \quad
  \forall x, y \in \Dfs, \;
   \bigg| \frac{1}{y} - \frac{1}{x} \bigg|  \leq \frac{1}{L} |{f^*}'(x) - {f^*}'(y)|.
\end{equation*}
This can be rewritten into the following integrated form
\begin{equation*}
  \text{$f$ is $L$-$\log$ smooth}
  \quad \Leftrightarrow \quad
  \forall x, y \in \Dfs, \;
   \Big| \int_y^x t^{-2} dt \Big|  \leq \frac{1}{L} \Big| \int_y^x ({f^*}''(t) dt \Big|,
\end{equation*}
which becomes equivalent to the desired result with the fundamental theorem of calculus
\begin{equation*}
\text{$f$ is $L$-$\log$ smooth}
  \quad \Leftrightarrow \quad
\forall x \in \Dfs, \; x^{-2} \leq  \frac{1}{L} {f^*}''(x).
\end{equation*}

\subsection{Inequalities for log-smooth functions} % (fold)
\label{sub:convex_inequality}

The proof of SDCA \cite{shalev2013stochastic} relies on the smoothness of the functions $f_i$
which implies strong convexity of their Fenchel conjugates $f_i^*$.
Indeed, a $\gamma$ strongly convex function $f^*$ satisfies the following inequality
\begin{equation}
  \label{eq:strongly-convex-barycentre-inequality}
      s f^*(x) + (1-s) f^*(y)
      \geq f^*(s x + (1 - s) y) + \frac{\gamma}{2} s (1 - s) (y - x)^2.
\end{equation}
This inequality is not satisfied for $L$-$\log$ smooth functions.
However, we can derive for such functions another inequality which can be compared to
such inequalities based on self-concordance and strongly convex properties.
\begin{lemma}
  \label{lemma:first_order_inequality}
  Let $f : \Df \subset \R \rightarrow \R$ be a strictly monotone convex function and $f^*$ be its
  differentiable Fenchel conjugate.
  Then,
  \begin{equation*}
    \text{$f$ is $L$-$\log$ smooth}
    \quad \Leftrightarrow \quad
    \forall x, y \in \Dfs, \;
   ( {f^*}'(x) - {f^*}'(y))(x - y) \geq L \frac{(x - y)^2}{xy} .
  \end{equation*}
  This bound is an equality for $f(x) = -L \log x$.
\end{lemma}
\begin{proof}
  From $\log$ smoothness definition, we obtain by multiplying both sides by $|f'(x) - f'(y)|$ and
  dividing by $f'(x)f'(y) > 0$ (since $f$ is strictly monotone),
  \begin{equation*}
  \text{$f$ is $L$-$\log$ smooth}
  \quad \Leftrightarrow \quad
  \forall x, y \in \Df, \;
  \frac{\big(f'(x) - f'(y) \big)^2}{f'(x)f'(y)} \leq \frac{1}{L} | x - y | | f'(x) - f'(y) |
  \end{equation*}
  Since $f$ is a convex function, $(f'(x) - f'(y))(x - y) \geq 0$ and using Proposition~
  \ref{prop:convex-conjugate-reciprocity}, we can rewrite the previous equivalence
  in the dual space:
  \begin{equation*}
  \text{$f$ is $L$-$\log$ smooth}
  \quad \Leftrightarrow \quad
  \forall x, y \in \Dfs, \;
  \frac{(x - y)^2}{xy} \leq \frac{1}{L} ({f^*}'(x) - {f^*}'(y)) (x - y),
  \end{equation*}
  which concludes the proof.
\end{proof}

\begin{lemma}
  \label{lemma:zero_order_inequality}
  Let $f : \Df \subset \R \rightarrow \R$ be a strictly monotone convex function and $f^*$ be its
  differentiable Fenchel conjugate.
  Then,
  \begin{equation*}
    \text{$f$ is $L$-$\log$ smooth}
    \quad \Leftrightarrow \quad
      \forall x, y \in \Dfs, \;
    {f^*}(x) - {f^*}(y) - {f^*}'(y)(x - y)
    \geq L \Big( \frac{x}{y} - 1 - \log \frac{x}{y} \Big).
  \end{equation*}
  This bound is an equality for $f(x) = - L \log x$.
\end{lemma}
\begin{proof}
Let $x, y \in \Dfs$, we have the following on the one hand,
\begin{equation*}
{f^*}(x) - {f^*}(y) - {f^*}'(y)(x - y)
= \int_y^x \big( {f^*}'(u) - {f^*}'(y) \big) \dif u
\end{equation*}
On the other hand, applying Lemma~\ref{lemma:first_order_inequality} together with the fundamental
theorem of calculus gives
\begin{equation*}
\text{$f$ is $L$-$\log$ smooth}
  \quad \Leftrightarrow \quad
  \forall x, y \in \Dfs, \;
\int_y^x ( {f^*}'(u) - {f^*}'(y)) \dif u
  \geq \int_y^x L \; \frac{u - y}{uy} \dif u.
\end{equation*}
Finally, solving the integral leads to the desired result:
\begin{equation*}
\forall x, y \in \Dfs, \;
\int_y^x L \; \frac{u - y}{uy} \dif u
  =  L \int_y^x \bigg( \frac{1}{y} - \frac{1}{u }\bigg) \dif u
  = L \bigg( \frac{x - y}{y} - \log \frac{x}{y} \bigg).
\end{equation*}
\end{proof}

\begin{lemma}
  \label{lemma:barycentre_inequality}
  Assume that $f$ is $L$-$\log$ smooth and $f^*$ is its differentiable Fenchel conjugate,
  then,
  \begin{equation*}
    s {f^*}(x) + (1 - s) {f^*}(y) - {f^*}(y + s (x - y))
    \geq L \bigg( \log \Big(1 - s + s \frac{x}{y} \Big) - s \log \frac{x}{y} \bigg)
  \end{equation*}
  for any $y, x \in \Dfs$ and $s \in [0, 1]$.
  This bound is an equality for $f(x) = - L \log x$.
\end{lemma}
\begin{proof}
Let $x, y \in \Dfs$ and define for any $s \in [0, 1]$,
$ u(s) = s x + (1 - s) y$.
We apply Lemma~\ref{lemma:zero_order_inequality} twice for $x$, $u(s)$ and $y$, $u(s)$:
\begin{equation}
  \label{eq:zero-order-in-x}
  f^*(x) - f^*(u(s)) - {f^*}'(u(s))(x - u(s))
  \geq L \Big( \frac{x}{u(s)} - 1 - \log \frac{x}{u(s)} \Big),
\end{equation}
\begin{equation}
  \label{eq:zero-order-in-y}
  f^*(y) - f^*(u(s)) - {f^*}'(u(s))(y - u(s))
  \geq L \Big( \frac{y}{u(s)} - 1 - \log \frac{y}{u(s)} \Big).
\end{equation}
Combining $s$\eqref{eq:zero-order-in-x} and $(1 - s)$\eqref{eq:zero-order-in-y} leads to
\begin{align*} s f^*(x) + (1 - s) f^*(y) - f^*(u(s))
    &\geq - s L \log \frac{x}{u(s)} - (1 - s) L \log \frac{y}{u(s)} \\
    & \qquad + L\bigg( s\frac{x}{u(s)} + (1 - s) \frac{y}{u(s)} - 1 \bigg) \\
    &= sL \log \frac{u(s)}{x} + (1 - s) L \log \frac{u(s)}{y}\\
    % &= s \log \frac{u(s)}{x} + (1 - s) \log \frac{u(s)}{x}  + (1 - s) \log \frac{x}{y} \\
    % &= \log \Big(s + (1 - s)\frac{y}{x} \Big) + (1 - s) \log \frac{x}{y}. \\
    &= s L \log \frac{u(s)}{y} + s L \log \frac{y}{x}  + (1 - s) L \log \frac{u(s)}{y}  \\
    &= L \log \Big(1 - s + s\frac{x}{y} \Big) + s L \log \frac{y}{x}. \qedhere
\end{align*}
\end{proof}
This Lemma which implies the barycenter $u(s) = y + s (x - y)$ for
$s \in [0, 1]$ is the lower bound that we actually use in proof of
Theorem~\ref{th:general-convergence}.
To compare this result with strong convexity and self-concordance assumptions, we will suppose that
$f^*$ is twice differentiable and hence that Proposition~\ref{prop:hyp_dual_hessian_equivalence}
applies.

\paragraph{Comparison with self-concordant functions} % (fold)
\label{par:comparison_with_self_concordant_functions}

Instead of building our lower bounds on $\log$ smoothness, we rather exhibit what can be obtained
with self-concordance combined with the lower bound ${{f^*}''(y) \geq L y^{-2}}$ from
Proposition~\ref{prop:hyp_dual_hessian_equivalence}.
In this paragraph, we consider that $\frac{1}{L} f^*$ is standard self-concordant, such an
hypothesis is verified for $f : t \mapsto - \frac{1}{L} \log(t)$.
Hence, using lower bound~\eqref{eq:self-concordant-first-order-lower-bound} on $\frac{1}{L} f$ and
then Proposition~\ref{prop:hyp_dual_hessian_equivalence}, we obtain
\begin{equation*}
  \forall x, y \in \Dfs, ~
  ({f^*}'(x) - {f^*}'(y))(x - y)
      \geq \frac{{f^*}''(y)(x - y)^2}{1 + \sqrt{\frac{1}{L} {f^*}''(y)}|x - y|}
      \geq L \frac{(x - y)^2}{y^2 + |y (x - y)|}.
\end{equation*}
Since $\forall x, y \in \Dfs$, $xy > 0$, this lower bound is equivalent to
Lemma~\ref{lemma:first_order_inequality} if $|x| \geq |y|$ but not as good otherwise.
Lemma~\ref{lemma:zero_order_inequality} can also be compared to what can be obtained applying
Inequality~\eqref{eq:self-concordant-zero-order-lower-bound} on $\frac{1}{L} f$.
Since $\omega : t \mapsto t - \log(1 + t)$ is an increasing function, it leads to
\begin{equation*}
  \forall x, y \in \Dfs, ~
  f^*(x) - f^*(y) - {f^*}'(y) (x - y)
       \geq L \; \omega \Big( \sqrt{\tfrac{1}{L} {f^*}''(y)} \; |x - y|\Big)
       \geq L \; \omega \Big( \big| \tfrac{x}{y} - 1 \big| \Big).
\end{equation*}
Again, this lower bound the same as Lemma~\ref{lemma:zero_order_inequality} if $|x| \geq |y|$
but not as good otherwise.
Finally, a bound equivalent to Lemma~\ref{lemma:barycentre_inequality} for self-concordant functions
is not easy to explicit in a clear form.
However, it is numerically smaller than the lower bound stated in
Lemma~\ref{lemma:barycentre_inequality} for
any $s \in [0, 1]$ and any $x, y \in \Dfs$.

% paragraph comparison_with_self_concordant_functions (end)

\paragraph{Comparison with strongly convex functions} % (fold)
\label{par:comparison_with_strongly_convex_functions}

We cannot directly assume that $f^*$ is strongly convex as it would mean that $f$ is
gradient-Lipschitz.
But, for fixed values of $x$ and $y$ $\in \Dfs$, we define on
$\{u \in \Dfs, |u| < \max(|x|, |y|) \}$ the function $f^*_{\{x, y\}} : u \mapsto f^*(u)$
as the restriction of $f^*$ on this interval.
The lower bound ${{f^*}''(y) \geq L y^{-2}}$ from
Proposition~\ref{prop:hyp_dual_hessian_equivalence} implies that $f^*_{\{x, y\}}$ is
$L / \max(x^2, y^2)$ strongly-convex on its domain to which $x$ and $y$ belong.
Equation~\eqref{eq:strong-convex-first-order-lower-bound} leads to the following inequality, valid
for $f^*_{\{x, y\}}$ and thus for $f^*$,
\begin{equation*}
  \forall x, y \in \Dfs, ~
  ({f^*}'(x) - {f^*}'(y))(x - y) \geq L \frac{(x - y)^2}{\max(x^2, y^2)}.
\end{equation*}
As soon as $x \neq y$, this lower bound is not as good as the one provided by
Lemma~\ref{lemma:first_order_inequality}.
Following the same logic, we exhibit the two following lower bounds.
The first one corresponds to
Lemma~\ref{lemma:zero_order_inequality} and is entailed by
Equation~\eqref{eq:strong-convex-zero-order-lower-bound},
\begin{equation}
\forall x, y \in \Dfs, ~
f^*(x) - f^*(y) - {f^*}'(y) (x - y) \geq \frac{L}{2} \frac{(x - y)^2}{\max(x^2, y^2)},
\end{equation}
and the second to Lemma~\ref{lemma:barycentre_inequality} and is entailed by
Equation~\eqref{eq:strongly-convex-barycentre-inequality}
\begin{equation*}
\forall x, y \in \Dfs, ~
\forall s \in [0, 1], ~
s {f^*}(x) + (1 - s) {f^*}(y) - {f^*}(y + s (x - y))
    \geq s (1 - s) \frac{L}{2} \frac{(x - y)^2}{\max(x^2, y^2)}.
\end{equation*}
In both cases the reached lower bounds are not as tight as the ones stood in
Lemmas~\ref{lemma:zero_order_inequality} and~\ref{lemma:barycentre_inequality}.
All these bounds are reported in Table~\ref{tab:convex-comparison} for an easy comparison.

\begin{table}
  \small
  \centering
  \begin{tabular*}{\textwidth}{c@{\extracolsep{\fill}}ccc}
    \toprule
    & strongly convex & self-concordant & $\log$ smoothness \\
    \midrule
    % $f'(y) - f'(x) (y - x)$
    Lemma~\ref{lemma:first_order_inequality}
        & $\frac{(x - y)^2}{\max(x^2, y^2)}$
        & $\frac{(x - y)^2}{y^2 + |y (x - y)|}$
        & $\frac{(x - y)^2}{xy}$ \\
    %$f(y) - f(x) - f'(x)(y - x)$
    Lemma~\ref{lemma:zero_order_inequality}
        & $\frac{(x - y)^2}{2 \max(x^2, y^2)}$
        & $\big| \tfrac{x}{y} - 1 \big| - \log \big( 1 + \big| \tfrac{x}{y} - 1 \big| \big) $
        & $\frac{x}{y} - 1 - \log \big(\frac{x}{y} \big)$ \\
    %$(1 - s) f(y) + s f(x) - f(u(s))$
    %Barycenter $u(s)$
    Lemma~\ref{lemma:barycentre_inequality}
        & $s (1 - s) \frac{(x - y)^2}{2 \max(x^2, y^2)}$
        & $-$
        & $\log \big(1 - s + s \frac{x}{y} \big) + s \log \frac{y}{x}$\\
    % $\lim_{s\rightarrow 0} \frac{u(s)}{s}$
    %     & $\frac{(y - x)^2}{2 \max(x, y)^2}$
    %     & $\frac{(1 - \log(2)) (y - x)^2}{\max(x, y)^2}$
    %     & $-1 + \frac{x}{y} + \log \big( \frac{y}{x} \big)$\\
    Reached for $f = - \log$
        & \xmark & \xmark & \cmark \\
    \bottomrule
  \end{tabular*}
  \caption{Comparison of lower bounds obtained with different hypotheses.
    These lower bounds come from Lemmas~\ref{lemma:first_order_inequality},
    \ref{lemma:zero_order_inequality} and \ref{lemma:barycentre_inequality}.
    It shows that both the strongly-convex and self-concordant hypotheses are not enough to reach
    the inequality obtained under $\log$ smoothness.
    The inequality coming from Lemma~\ref{lemma:barycentre_inequality} is missing as it cannot be
    easily exhibited for self-concordant functions.}
  \label{tab:convex-comparison}
\end{table}

% paragraph comparison_with_strongly_convex_functions (end)

$~$\newline
Finally, two lemmas to lower bound Lemma~\ref{lemma:barycentre_inequality} are needed as well.

\begin{lemma}
\label{lemma:barycentre-lower-bound-ratio2-increasing}
The function $f$ defined by
\begin{equation*}
  f(s, z) = \frac{\log \big( (1 - s) + \frac{s}{z} \big) + s \log z}{(1 - z)^2}
\end{equation*}
for all $z \in \R^{++}$ and $s \in [0, 1]$ is a decreasing function in $z$.
\end{lemma}
% This Lemma is illustrated in Figure~\ref{fig:barycentre-lower-bound-ratio2-increasing}.
%
% \begin{figure}
% \centering
% \includegraphics[width=0.6\textwidth]{plot_barycentre_lower_bound_decreasing}
% \caption{Illustration of Lemma~\ref{lemma:barycentre-lower-bound-ratio2-increasing}. The plotted
% function is decreasing in $z$ for any fixed $s \in [0, 1]$.}
% \label{fig:barycentre-lower-bound-ratio2-increasing}
% \end{figure}
%
%
\begin{lemma}
  \label{lemma:barycentre-lower-bound-ratio}
  We have
  \begin{equation*}
    \log \Big( (1 - s) + \frac{s}{z} \Big) + s \log z
    \geq s (1 - s) \Big( \frac{1}{z} - 1 + \log z \Big)
  \end{equation*}
  for all $z \geq 1$ and $s \in [0, 1]$.
\end{lemma}
%
% This lemma is illustrated in Figure~\ref{fig:barycentre-lower-bound}.
%
% \begin{figure}
% \centering
% \includegraphics[width=0.55\textwidth]{plot_barycentre_lower_bound}
% \caption{Illustration of Lemma~\ref{lemma:barycentre-lower-bound-ratio} showing that for any
% $z \geq 1$,
% $\log \big( (1 - s) + \tfrac{s}{z} \big) + s \log z
% \geq s (1 - s) \big( \tfrac{1}{z} - 1 + \log z \big)$.}
% \label{fig:barycentre-lower-bound}
% \end{figure}
% subsection convex_inequality (end)
The analytical proof of these lemmas are very technical and not much informative.
Thus, we rather illustrate them with the two following figures
\begin{figure}
\centering
\includegraphics[width=0.45\textwidth]{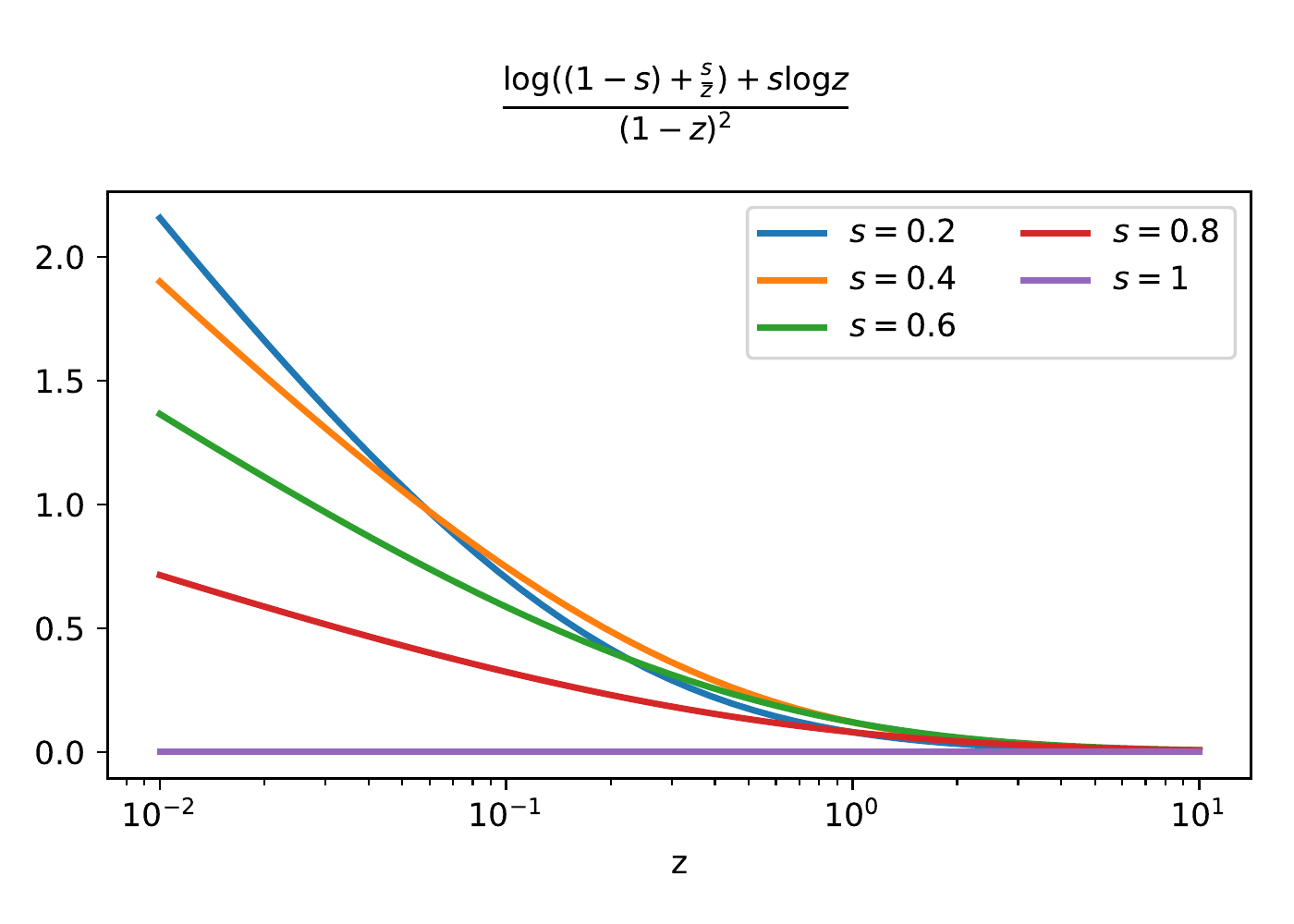}
\includegraphics[width=0.45\textwidth]{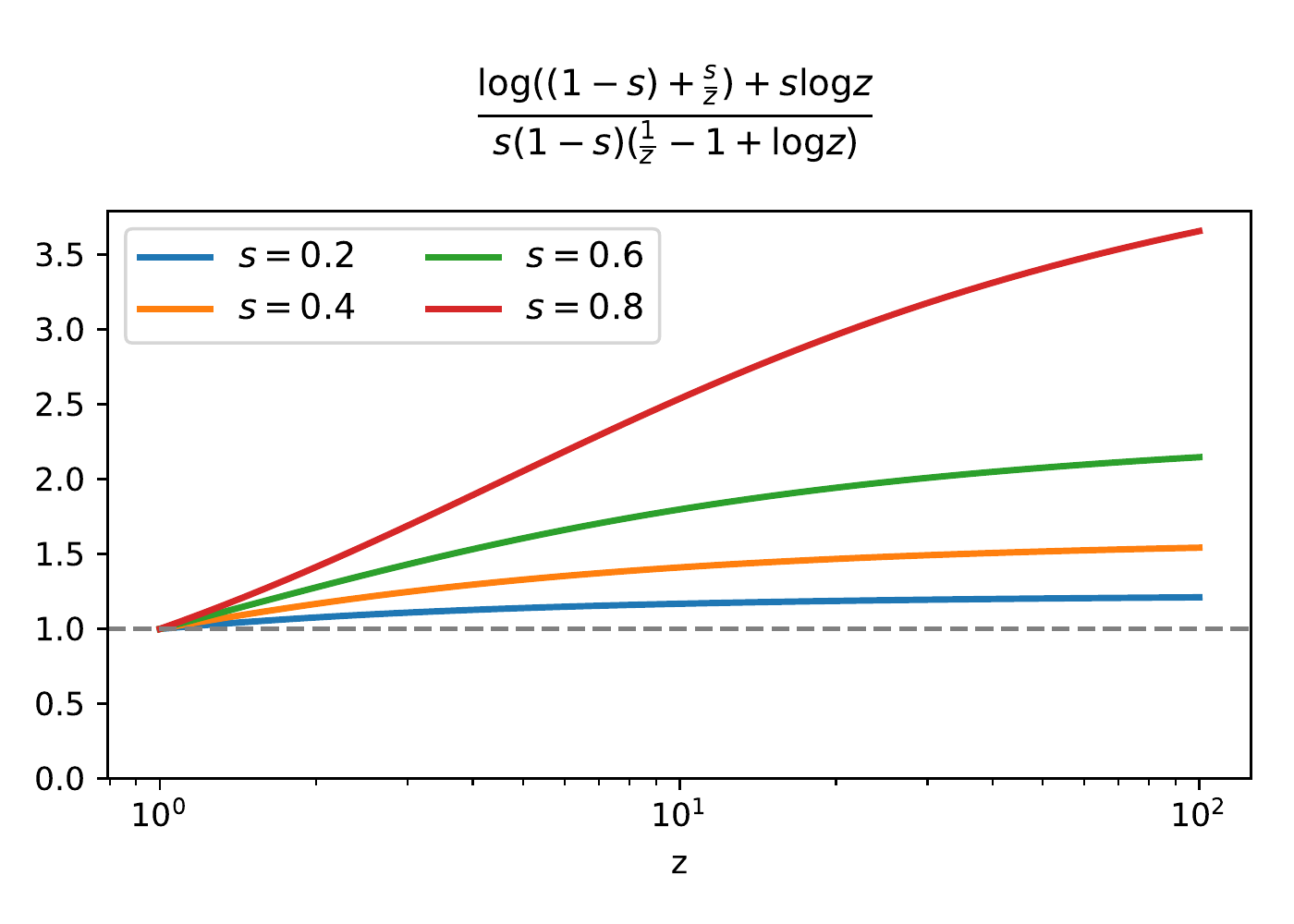}
\caption{Illustration of Lemma~\ref{lemma:barycentre-lower-bound-ratio} showing that for any
$z \geq 1$,
$\log \big( (1 - s) + \tfrac{s}{z} \big) + s \log z
\geq s (1 - s) \big( \tfrac{1}{z} - 1 + \log z \big)$.}
\label{fig:barycentre-lower-bound}
\end{figure}

\subsection{Proof of Theorem~\ref{th:general-convergence}} % (fold)
\label{sub:contraction}

This proof is very similar to SDCA's proof \cite{shalev2014accelerated} but it uses the new
convex inequality on the Fenchel conjugate of $\log$ smooth functions from Lemma~
\ref{lemma:barycentre_inequality} to get a tighter inequality.
We first prove the following lemma which is an equivalent of Lemma~6 from
\cite{shalev2014accelerated} but with convex functions $f_i$ that are $L_i$-$\log$ smooth
instead of being $L_i$-gradient Lipschitz.
\begin{lemma}
\label{lemma:weighted_dual_gain}
  Suppose that we known bounds $\beta_i \in \Dfsm$ such that
  $R_i = \beta_i / \alpha_i^* \geq 1$ for $i=1,\ldots,n$
  and assume that all $f_i$ are $L_i$-$\log$ smooth with differentiable Fenchel conjugates
  and that $g$ is 1-strongly convex.
  Then, if $\alpha^{(t, i)}$ is the value of $\alpha^{(t)}$ when $i$ is sampled at
  iteration $t$ for Algorithms~\ref{alg:general-shifted-sdca} and \ref{alg:prox-shifted-sdca}, we
  have
  \begin{equation}
    \label{eq:dual_gain}
    \sum_{i=1}^n s_i^{-1} \big( D(\alpha^{(t, i)}) - D(\alpha^{(t - 1)}) \big)
     \geq D(\alpha^*) - D(\alpha^{(t - 1)}) + G(s_i, \alpha_i^{(t-1)}, \alpha_i^*)
  \end{equation}
  for any ${s_1, \dots, s_n} \in [0, 1]$, where
  \begin{equation*}
  G(s, \alpha^{(t-1)}, \alpha^*) = \frac{1}{n} \sum_{i = 1}^{n} \Big(
                    L_i \gamma(s_i, \alpha_i^{(t-1)}, \alpha_i^*) -
                    \frac{s_i}{2 \lambda n} \|x_i\|^2  (\alpha_i^* - \alpha_i^{(t-1)})^2
                \Big)
  \end{equation*}
  and
  \begin{equation*}
  \gamma(s_i, \alpha_i^{(t-1)}, \alpha_i^*) =
  % \frac{1}{s_i} \log \Bigg(s_i + (1 - s_i) \frac{\alpha_i^{(t-1)}}{\alpha_i^*} \Bigg) -
  % \frac{1 - s_i}{s_i} \log \Bigg(\frac{\alpha_i^{(t-1)}}{\alpha_i^*} \Bigg)
  \frac{1}{s_i} \log \Big( 1 - s_i + s_i \frac{\alpha_i^*}{\alpha_i^{(t-1)}} \Big)
  - \log \frac{\alpha_i^*}{\alpha_i^{(t-1)}}.
  \end{equation*}
\end{lemma}

\begin{proof}
At iteration $t$, if the dual vector is set to $\alpha^{(t)}$ (and
$v^{(t)} = v(\alpha^{(t)})$, see Equation~\eqref{eq:primal_dual_relation}) the dual gain is
\begin{equation*}
n (D(\alpha^{(t)}) - D(\alpha^{(t-1)})) =
    \underbrace{\big( -f_i^* (-\alpha_i^t) - \lambda n g^*(v^{(t)}) \big)}_{A_i}
    - \underbrace{\big( -f_i^* (-\alpha_i^{(t-1)}) - \lambda n g^*(v^{(t-1)}) \big)}_{B_i}
\end{equation*}
where $i$ is the index sampled at iteration $t$ (see Line~\ref{lst:line:random}).
For Algorithm~\ref{alg:general-shifted-sdca}, by the definition of $\alpha_i^{(t)}$ given on
Lines~\ref{lst:line:general_local_max} and~\ref{lst:line:general-bounding} we have
\begin{equation*}
A_i = \max_{\substack{\alpha_i \in \Dfsm \\ \text{s.t. } \beta_i / \alpha_i \geq 1}}
      - f_i^*(-\alpha_i) - \lambda n g^*\Big(
        \frac{1}{\lambda n} (\alpha_i - \alpha_i^{(t-1)}) x_i
        + \frac{1}{\lambda n} \sum_{j = 1}^n \alpha_j^{(t-1)} x_j
        - \frac{1}{\lambda} \psi
    \Big).
\end{equation*}
Using the smoothness inequality on $g^*$ which is 1-smooth as $g$ is 1-strongly convex,
\begin{multline*}
g^*(v^{(t-1)} + \Delta v) \leq h(v^{(t-1)}, \Delta v) \\ \text{where} \;\;
h(v^{(t-1)}, \Delta v) = g^*(v^{(t-1)}) + \nabla g^*(v^{(t-1)})^\top \Delta v + \frac{1}{2} \|\Delta v\|^2.
\end{multline*}
Hence setting $\Delta v$ to $(\lambda n)^{-1} (\alpha_i - \alpha_i^{(t-1)}) x_i$
% and given
%Relation~\eqref{eq:primal_dual_relation} linking $v$ to $\alpha$ and $\psi$
, we can lower bound $A_i$ with
\begin{equation*}
A_i \geq
    \max_{\substack{\alpha_i \in \Dfsm \\ \text{s.t. } \beta_i / \alpha_i \geq 1}}
     - f_i^* (- \alpha_i)
     - \lambda n h \big(
        v^{(t-1)}, (\lambda n)^{-1} (\alpha_i - \alpha_i^{(t-1)}) x_i
    \big).
\end{equation*}
For Algorithm~\ref{alg:prox-shifted-sdca}, by definition of $\alpha_i^{(t)}$ stated at
Lines~\ref{lst:line:local_max} and~\ref{lst:line:bounding} combined with the modified argmax
relation~\eqref{eq:modified-argmax},
\begin{equation*}
A_i = \max_{\substack{\alpha_i \in \Dfsm \\ \text{s.t. } \beta_i / \alpha_i \geq 1}}
       - f_i^* (- \alpha_i)
       - \lambda n
         h \big( v^{(t-1)}, (\lambda n)^{-1} (\alpha_i - \alpha_i^{(t-1)}) x_i \big).
\end{equation*}
As both $\alpha_i^{(t-1)}$ and $\alpha_i^*$ belong to
$\{\alpha_i \in \Dfsm, \beta_i / \alpha_i \geq 1\}$, for any $s_i \in [0, 1]$, the convex
combination $\alpha_i = (1 - s_i)\alpha_i^{(t-1)} + s_i \alpha_i^*$ belongs to it as well.
Hence, for both algorithms, $A_i$ is higher than the previous quantity evaluated at this specific
$\alpha_i$.
Namely,
\begin{equation*}
A_i \ge
    - f_i^* \big(- \big( (1 - s_i)\alpha_i^{(t-1)} + s_i \alpha_i^* \big) \big)
    - \lambda n h \big(v^{(t-1)}, (\lambda n)^{-1} s_i(\alpha_i^* - \alpha_i^{(t-1)}) x_i \big).
\end{equation*}
We then use Lemma~\ref{lemma:barycentre_inequality},
in which $-\alpha_i^* \in \Dfs$ stands for $x$ and $-\alpha_i^{(t-1)} \in \Dfs$
for $y$:
\begin{equation*}
  (1 - s_i) f_i^*(- \alpha_i^{(t-1)}) + s_i f_i^*(-\alpha_i^*)
  - f_i^*(- (1 - s_i) \alpha_i^{(t-1)} - s_i \alpha_i^*)
  \geq s_i L_i \gamma(s_i, \alpha_i^{(t-1)}, \alpha_i^*)
\end{equation*}
where
\begin{equation*}
  \gamma(s_i, \alpha_i^{(t-1)}, \alpha_i^*) =
  \frac{1}{s_i} \log \Big( 1 - s_i + s_i \frac{\alpha_i^*}{\alpha_i^{(t-1)}} \Big)
  - \log \frac{\alpha_i^*}{\alpha_i^{(t-1)}}.
\end{equation*}
This inequality is used instead of the strong convex inequality of the classic SDCA analysis
\cite{shalev2014accelerated}. If we plug this inequality into $A_i$ we obtain
\begin{align*}
A_i
    %&\geq  - s_i f^*(-\alpha_i^*) - (1-s_i) f^*(-\alpha_i^{(t-1)})
    %       + \frac{\gamma_i^{(t-1)}}{2} s_i (1 - s_i) (\alpha_i^{(t-1)} - \alpha_i^*)^2 \\
    %&\quad \quad
    % - \lambda n h \Big(v^{(t-1)}, (\lambda n)^{-1} s_i(\alpha_i^*
    % - \alpha_i^{(t-1)}) x_i \Big) \\
    &\geq   - s_i f^*(-\alpha_i^*) - (1 - s_i) f^*(-\alpha_i^{(t-1)})
          + s_i L_i \gamma(s_i, \alpha_i^{(t-1)}, \alpha_i^*)\\
    &\quad\quad
          - \lambda n g^* (v^{(t-1)})
          - s_i(\alpha_i^* - \alpha_i^{(t-1)}) x_i^\top \nabla g^*(v^{(t-1)})
          - \frac{s_i^2 (\alpha_i^* - \alpha_i^{(t-1)})^2}{2 \lambda n} \|x_i\|^2 \\
    & =  - s_i ( f^*(-\alpha_i^*) - f^*(-\alpha_i^{(t-1)}) )
          - f^*(-\alpha_i^{(t-1)}) - \lambda n g^*(v^{(t-1)})\\
    &\quad\quad - s_i(\alpha_i^* - \alpha_i^{(t-1)}) x_i^\top \nabla g^*(v^{(t-1)})
           + s_i \Big( L_i \gamma(s_i, \alpha_i^{(t-1)}, \alpha_i^*) -
             \frac{s_i}{2 \lambda n} \|x_i\|^2  (\alpha_i^* - \alpha_i^{(t-1)})^2 \Big).
\end{align*}
Hence, we retrieve $B_i$ and rewrite the previous inequality as
\begin{align*}
s_i^{-1} (A_i - B_i) &\geq  - \big( f^*(-\alpha_i^*) - f^*(-\alpha_i^{(t-1)}) \big)
    - (\alpha_i^* - \alpha_i^{(t-1)}) x_i^\top \nabla g^*(v^{(t-1)})\\
    &\quad\quad
           + L_i \gamma(s_i, \alpha_i^{(t-1)}, \alpha_i^*) -
             \frac{s_i}{2 \lambda n} \|x_i\|^2  (\alpha_i^* - \alpha_i^{(t-1)})^2.
\end{align*}
We can sum over all possible sampled $i$ and weight each entry with $s_i^{-1}$ to obtain
\begin{align}
\label{eq:dual_gain_step}
\sum_{i = 1}^{n} s_i^{-1} (A_i - B_i) &\geq
         - \sum_{i = 1}^{n} \Big( f^*(-\alpha_i^*) - f^*(-\alpha_i^{(t-1)}) \Big)
               -  \bigg\langle \nabla g^*(v^{(t-1)}) \; \Big| \;
                   \sum_{i=1}^n (\alpha_i^* - \alpha_i^{(t-1)}) x_i \bigg\rangle
               \nonumber \\
        &\quad\quad
               + \sum_{i = 1}^{n} \Big(
                    L_i \gamma(s_i, \alpha_i^{(t-1)}, \alpha_i^*) -
                    \frac{s_i}{2 \lambda n} \|x_i\|^2  (\alpha_i^* - \alpha_i^{(t-1)})^2
                \Big).
\end{align}
Then since $g^*$ is convex, we obtain
\begin{align*}
\bigg\langle \nabla g^*(v^{(t-1)}) \; \Big| \;
     \sum_{i=1}^n (\alpha_i^* - \alpha_i^{(t-1)}) x_i \bigg\rangle
           &= \big\langle \nabla g^*(v^{(t-1)})  \; | \;
                  \lambda n (v(\alpha^*) - v^{(t-1)}) \big\rangle \\
           &\leq \lambda n  \big( g^*(v(\alpha^*)) - g^*(v^{(t-1)}) \big),
\end{align*}
which can be injected in Equation~\eqref{eq:dual_gain_step} leading to
\begin{align*}
\sum_{i = 1}^{n} s_i^{-1} (A_i - B_i) &\geq
         - \sum_{i = 1}^{n} \Big( f^*(-\alpha_i^*) - f^*(-\alpha_i^{(t-1)}) \Big)
               + \lambda n g^*(v(\alpha^*)) - \lambda n g^*(v^{(t-1)}) \\
        &\quad\quad
               + \sum_{i = 1}^{n} \Big(
                    L_i \gamma(s_i, \alpha_i^{(t-1)}, \alpha_i^*) -
                    \frac{s_i}{2 \lambda n} \|x_i\|^2  (\alpha_i^* - \alpha_i^{(t-1)})^2
                \Big).
\end{align*}
Finally, since $A_i - B_i = n ( D(\alpha^{(t, i)}) - D(\alpha^{(t-1)}) )$, we obtain
\begin{multline*}
\sum_{i=1}^n s_i^{-1} \big( D(\alpha^{(t, i)}) - D(\alpha^{(t - 1)}) \big) \\
     \geq D(\alpha^*) - D(\alpha^{(t - 1)})
          + \frac{1}{n} \sum_{i = 1}^{n} \Big(
                    L_i \gamma(s_i, \alpha_i^{(t-1)}, \alpha_i^*) -
                    \frac{s_i}{2 \lambda n} \|x_i\|^2  (\alpha_i^* - \alpha_i^{(t-1)})^2
                \Big).
\end{multline*}
This concludes the proof of Lemma~\ref{lemma:weighted_dual_gain}.
\end{proof}

From Lemma~\ref{lemma:weighted_dual_gain}, we obtain a contraction speed as soon as
$G(s, \alpha^{(t-1)}, \alpha^*) \geq 0$.
If $\alpha_i^{(t-1)} \neq \alpha_i^*$ this is obtained if
\begin{equation}
  \label{eq:convergence-condition}
  \forall i \in \{1, \dots, n\}, \quad
  L_i \gamma(s_i, \alpha_i^{(t-1)}, \alpha_i^*) -
                    \frac{s_i}{2 \lambda n} \|x_i\|^2  (\alpha_i^* - \alpha_i^{(t-1)})^2 \geq 0
\end{equation}
\begin{equation*}
  \Leftrightarrow \quad
  \forall i \in \{1, \dots, n\}, \quad
  \frac{\gamma(s_i, \alpha_i^{(t-1)}, \alpha_i^*)}
       {\Big(1 - \frac{\alpha_i^{(t-1)}}{\alpha_i^*} \Big)^2} -
    s_i  \frac{ \|x_i\|^2{\alpha_i^*}^2}{2 \lambda n L_i} \geq 0.
\end{equation*}
By definition of $\gamma$ we have
\begin{equation*}
\frac{\gamma(s_i, \alpha_i^{(t-1)}, \alpha_i^*)}
     {\Big(1 - \frac{\alpha_i^{(t-1)}}{\alpha_i^*} \Big)^2}
=
% \frac{\log \Big(s_i + (1 - s_i) \frac{\alpha_i^{(t-1)}}{\alpha_i^*} \Big)
% - (1 - s_i) \log \Big(\frac{\alpha_i^{(t-1)}}{\alpha_i^*} \Big)}
  \frac{\log \Big( 1 - s_i + s_i \frac{\alpha_i^*}{\alpha_i^{(t-1)}} \Big)
  - s_i \log \frac{\alpha_i^*}{\alpha_i^{(t-1)}}}
       {s_i \Big(1 - \frac{\alpha_i^{(t-1)}}{\alpha_i^*} \Big)^2}.
\end{equation*}
As $\alpha_i^{(t-1)} / \alpha_i^*$ is bounded by $\beta_i / \alpha_i^*$, we apply
Lemma~\ref{lemma:barycentre-lower-bound-ratio2-increasing} to obtain
\begin{equation*}
\frac{\gamma(s_i, \alpha_i^{(t-1)}, \alpha_i^*)}
     {\Big(1 - \frac{\alpha_i^{(t-1)}}{\alpha_i^*} \Big)^2}
\geq
    % \frac{\log \Big(s_i + (1 - s_i) \frac{\beta_i}{\alpha_i^*} \Big)
    % - (1 - s_i) \log \Big(\frac{\beta_i}{\alpha_i^*} \Big)}
    \frac{\log \Big( 1 - s_i + s_i \frac{\alpha_i^*}{\beta_i} \Big)
          - s_i \log \frac{\alpha_i^*}{\beta_i}}
       {s_i \Big(1 - \frac{\beta_i}{\alpha_i^*} \Big)^2},
\end{equation*}
and as $\beta_i / \alpha_i^* \geq 1$, we can apply
Lemma~\ref{lemma:barycentre-lower-bound-ratio} leading to
\begin{equation*}
\frac{\gamma(s_i, \alpha_i^{(t-1)}, \alpha_i^*)}
     {\Big(1 - \frac{\alpha_i^{(t-1)}}{\alpha_i^*} \Big)^2}
\geq \frac{\big( 1 - s_i \big) \Big( \frac{\alpha_i^*}{\beta_i}
           - 1 + \log \frac{\beta_i}{\alpha_i^*} \Big)}
          {\Big(1 - \frac{\beta_i}{\alpha_i^*} \Big)^2}.
\end{equation*}
Finally the convergence condition from Equation~\eqref{eq:convergence-condition} is
satisfied when
\begin{equation*}
  \forall i \in \{1, \dots, n\}, \quad
  \big( 1 - s_i \big) \bigg( \frac{\alpha_i^*}{\beta_i} - 1 + \log \frac{\beta_i}{\alpha_i^*} \bigg)
  -
    %s_i  \frac{(\beta_i - \alpha_i^*)^2 \|x_i\|^2}{2 \lambda n L_i}
    s_i  \frac{ \|x_i\|^2{\alpha_i^*}^2}{2 \lambda n L_i}
    \bigg(1 - \frac{\beta_i}{\alpha_i^*} \bigg)^2
    \geq 0
\end{equation*}
which is true for any $s_i \in [0, \sigma_i]$ where
\begin{equation*}
  \sigma_i
  = \left( 1 +
      \frac{\|x_i\|^2  {\alpha_i^*}^2}{2 \lambda n L_i}
      \frac{\Big(1 - \frac{\beta_i}{\alpha_i^*} \Big)^2}
                    {\frac{\alpha_i^*}{\beta_i}
                       + \log \frac{\beta_i}{\alpha_i^*} - 1
                    }
    \right)^{-1}.
\end{equation*}
Theorem~\ref{th:general-convergence} is obtained by sampling uniformly $i$, meaning haing all
$s_i$ equal.
Hence, to fulfill Equation~\eqref{eq:convergence-condition}, we set
\begin{equation}
  \label{eq:s-uniform}
  s_i = \min_{j \in \{1, \dots, n\}}  \sigma_j
\end{equation}
for all $i \in \{1, \dots, n\}$.
We then lower bound the expectation of $D(\alpha^{(t)}) - D(\alpha^{(t-1)})$ over all possible
sampled $i$ and obtain
\begin{equation*}
\E [D(\alpha^{(t)}) - D(\alpha^{(t-1)})]
    = \frac{1}{n} \sum_{i=1}^n D(\alpha^{(t, i)}) - D(\alpha^{(t-1)})
    \geq \frac{\min_j  \sigma_j}{n}(D(\alpha^*) - D(\alpha^{(t - 1)})),
\end{equation*}
by multiplying Equation~\eqref{eq:dual_gain} with $\min_{j \in \{1, \dots, n\}}  \sigma_j / n$ and
removing the quantity $G^{(t-1)} \geq 0$.
This leads to the following convergence speed after $t$ iterations,
\begin{equation*}
    \E [D(\alpha^*) - D(\alpha^{(t)})]
    \leq \Big( 1 - \frac{\min_j  \sigma_j}{n} \Big)^t ( D(\alpha^*) - D(\alpha^{(0)}) ),
\end{equation*}
which concludes the proof of Theorem~\ref{th:general-convergence}. $\hfill \qed$

% subsection contraction (end)

\subsection{Proof of Theorem~\ref{th:convergence-is}} % (fold)
\label{sub:theorem_th:convergence-is}

Instead of taking all $s_i$ equal as in the uniform sampling setting (see
Equation~\eqref{eq:s-uniform}), we rather parametrize
$s_i$ by $\frac{\bar{\sigma}}{\rho_i n}$ where $\rho_i$ is the probability of sampling $i$.
Then, we obtain the following expectation under $\rho$,
\begin{equation*}
\E_\rho [D(\alpha^{(t)}) - D(\alpha^{(t-1)})]
    = \sum_{i=1}^n \rho_i D(\alpha^{(t, i)}) - D(\alpha^{(t-1)}).
\end{equation*}
Since we have $\rho_i = \frac{n}{\bar{\sigma}} s_i^{-1} $ we obtain the
following inequality using Lemma~\ref{lemma:weighted_dual_gain}:
\begin{equation}
\label{eq:contraction-is-unconstrained}
  \E_\rho [D(\alpha^{(t)}) - D(\alpha^{(t-1)})]
    \geq \frac{\bar{\sigma}}{n} \Big(D(\alpha^*) - D(\alpha^{(t - 1)})
                         + G\big(\bar{\sigma} (\rho n)^{-1}, \alpha^{(t-1)}, \alpha^* \big) \Big).
\end{equation}
To ensure that $G\big(\bar{\sigma} (\rho n)^{-1}, \alpha^{(t-1)}, \alpha^* \big) \geq 0$ while
keeping the biggest gain, we must satisfy the constraint from
Equation~\eqref{eq:convergence-condition} and find feasible $\rho$ and $\bar{\sigma}$ that maximize
the following problem:
\begin{equation*}
\max_{\bar{\sigma}} \bar{\sigma}
\quad \text{subject to}
\quad \frac{\bar{\sigma}}{\rho_i n} \in [0, \sigma_i],
\quad \rho_i \geq 0,
\quad \sum_{i=1}^n \rho_i = 1.
\end{equation*}
This problem is solved by Proposition~1 of \cite{zhao2015stochastic} and leads to the following choices:
\begin{equation}
  \label{eq:is-optimal-parameters}
  \rho_i = \frac{\sigma_i^{-1}}{\sum_{j=1}^{n}\sigma_j^{-1}}
  \quad \text{and} \quad
  \bar{\sigma} = \Big(\frac{1}{n} \sum_{i=1}^{n} \sigma_i^{-1} \Big)^{-1}.
\end{equation}
This choice for $\rho$ and $\bar{\sigma}$ ensures
$G(\bar{\sigma} (\rho n)^{-1}, \alpha^{(t-1)}, \alpha^*) \geq 0$
hence Equation~\eqref{eq:contraction-is-unconstrained} without
$G\big(\bar{\sigma} (\rho n)^{-1}, \alpha^{(t-1)}, \alpha^* \big)$ leads to
\begin{equation*}
\E_\rho [D(\alpha^{(t)}) - D(\alpha^{(t-1)})]
    \geq \frac{\bar{\sigma}}{n}(D(\alpha^*) - D(\alpha^{(t - 1)})),
\end{equation*}
and finally, after $t$ iterations, we have
\begin{equation*}
    \E [D(\alpha^*) - D(\alpha^{(t)})]
    \leq \Big( 1 - \frac{\bar{\sigma}}{n} \Big)^t ( D(\alpha^*) - D(\alpha^{(0)}) ),
\end{equation*}
which concludes the proof of Theorem~\ref{th:convergence-is}. $\hfill \qed$

\subsection{Proof of Proposition~\ref{prop:closed_form_solution_log_losses}}
\label{sub:closed_for_solution_for_log_losses}

With $f_i^*(- \alpha_i) = -y_i - y_i \log \frac{\alpha_i}{y_i}$, let
\begin{equation*}
\phi(\alpha_i) = y_i + y_i \log \frac{\alpha_i}{y_i}
           - \frac{\lambda n}{2} \Big\|
             w^{(t-1)}
             + (\lambda n)^{-1} (\alpha_i - \alpha_i^{(t - 1)}) x_i
           \Big\|^2
\end{equation*}
be the function to optimize. Note that $\phi$ is a concave function from $\Dfsm$ to $\R$ and hence
it reaches its minimum if its gradient is zero:
\begin{equation*}
\phi'(\alpha_i) = \frac{y_i}{\alpha_i}
           - x_i^\top w^{(t-1)}
           - \frac{\|x_i \|^2}{\lambda n}  (\alpha_i - \alpha_i^{(t - 1)})  = 0.
\end{equation*}
This second order equation in $\alpha_i$ has a unique positive solution, the one stated in
Proposition~\ref{prop:closed_form_solution_log_losses}.

\subsection{Proof of Proposition~\ref{prop:dual_large_bound}} % (fold)
\label{sub:dual_variable_bound}

Given that we are using Ridge regularization, the values of $f_i^*$ and $g^*$ are
\begin{equation*}
f_i^* (v) = -y_i -y_i \log \Big(\frac{-v}{y_i}\Big)
\quad \text{and} \quad
g^*(w) = g(w) = \frac{1}{2} \|w\|^2.
\end{equation*}
Hence the conditions at optimum~\eqref{eq:general_primal_dual_relations_at_optimum_w}
and~\eqref{eq:general_primal_dual_relations_at_optimum_alpha} become
\begin{equation}
\label{eq:primal_dual_relations_at_optimum}
  w^* = \frac{1}{\lambda n} \sum_{i=1}^n \alpha_i^* x_i -\frac{1}{\lambda} \psi
  \quad \text{and} \quad
  \forall i \in \{1, \dots n\}, \quad \alpha_i^* = \frac{y_i}{{w^*}^\top x_i }.
\end{equation}
By combining both equations with Equation~\eqref{eq:primal_dual_relations_at_optimum}, we have
\begin{equation}
\label{eq:dual_optimum_relation}
\forall i \in \{1, \dots n\}, \quad \alpha_i^* =
   \frac{\lambda n y_i}{\sum_{j=1}^n \alpha_j^* x_j^\top x_i - n \psi^\top x_i}.
\end{equation}
Since the inner products $x_i^\top x_j$ and $\alpha_i$ are non-negative, we can remove the
terms $\sum_{j \neq i} \alpha_j^* x_j^\top x_i$ and upper bound the dual variable with
\begin{equation*}
\forall i \in \{1, \dots n\}, \quad \alpha_i^* \leq
   \frac{\lambda n y_i}{ \alpha_i^* \| x_i \|^2 - n \psi^\top x_i }.
\end{equation*}
By solving this second order inequality, we can derive the following upper bound for all $\alpha_i^*$:
\begin{equation*}
       \alpha_i^* \leq \frac{1}{2 \|x_i\|^2} \Big( n \psi^\top x_i
                + \sqrt{ ( n \psi^\top x_i )^2 +
                         4 \lambda n y_i \|x_i\|^2 }\Big),
\end{equation*}
which concludes the proof. $\hfill \qed$

% subsection dual_variable_bound (end)

% section dual_variable (end)

\subsection{Proof of Remark~\ref{rmk:closed-form-bounded}} % (fold)
\label{sec:proof_of_remark_rmk:closed-form-bounded}

At each iteration the closed form solution is given by
Proposition~\ref{prop:closed_form_solution_log_losses}:
\begin{equation*}
\alpha^t_i =
    \frac{1}{2} \Bigg(
           \sqrt{
              \Big(
                  \alpha_i^{(t - 1)}
                   - \frac{\lambda n}{\|x_i\|^2} x_i^\top w^{(t-1)}
              \Big)^2
               + 4 \lambda n \frac{y_i}{\|x_i\|^2}
            }
            + \alpha_i^{(t - 1)}
            - \frac{\lambda n}{\|x_i\|^2} x_i^\top w^{(t-1)}
    \Bigg).
\end{equation*}
Since the inner products $x_i^\top x_j$ are non-negative, we obtain
\begin{equation*}
\alpha_i^{(t - 1)} - \frac{\lambda n}{\|x_i\|^2} x_i^\top w^{(t-1)}
= n \frac{\psi^\top x_i }{\|x_i\|^2}
  - \sum_{j \neq i} \alpha_j^{(t - 1)} \frac{x_j^\top x_i }{\|x_i\|^2}
\leq n \frac{\psi^\top x_i }{\|x_i\|^2}
\end{equation*}
and since $\alpha^t_i$ is increasing with
${(\alpha_i^{(t - 1)} - \frac{\lambda n}{\|x_i\|^2} x_i^\top w^{(t-1)} )}$,
we obtain
\begin{equation*}
\alpha^t_i \leq
\frac{1}{2} \Bigg(
        \sqrt{ n \frac{\psi^\top x_i^2}{\|x_i\|^4}
               + 4 \lambda n \frac{y_i}{\|x_i\|^2} }
        + n \frac{\psi^\top x_i }{\|x_i\|^2}
\Bigg) = \beta_i,
\end{equation*}
which concludes the proof. $\hfill \qed$

\subsection{Proof of Proposition~\ref{prop:dual-optimum-linked-to-norm}} % (fold)
\label{sub:proposition_prop:dual-optimum-linked-to-norm}

This proposition easily follows from the following computation
\begin{align*}
\zeta^* &= \argmax_{\zeta \in (0, +\infty)^n}
      \frac{1}{n} \sum_{i=1}^n - y_i  - y_i \log \frac{\zeta_i}{y_i}
          - \lambda g^* \bigg(\frac{1}{\lambda n} \sum_{i=1}^n \zeta_i \xi_i -
          \frac{1}{\lambda} \psi \bigg) \\
      &= \argmax_{\zeta \in (0, +\infty)^n}
      \frac{1}{n} \sum_{i=1}^n - y_i  - y_i \log \frac{\zeta_i}{y_i}
          + y_i \log(c_i)
          - \lambda g^* \bigg(\frac{1}{\lambda n} \sum_{i=1}^n c_i \zeta_i x_i -
          \frac{1}{\lambda} \psi \bigg) \\
      &= \argmax_{\zeta \in (0, +\infty)^n} D(c \cdot \zeta),
\end{align*}
where $D$ is the original dual problem and $c \cdot \zeta$ is the element wise product of the
vectors $c$ and $\zeta$.
Then, since
\begin{equation*}
      \argmax_x \{ x \mapsto f(c x) \}
          = \frac{1}{c} \argmax \{ x \mapsto f(x) \},
\end{equation*}
which remains valid in the multivariate case, we obtain $\zeta^*_i = \alpha_i^* / c_i$ for any
$i = 1, \dots, n$. $\hfill \qed$

\subsection{Proof of Proposition~\ref{prop:dual-initial-guess-rescale}} % (fold)
\label{sub:proposition_prop:dual-initial-guess-rescale}

Using $\alpha^* = \bar{\alpha} \kappa$ the dual problem becomes one dimensional
\begin{equation*}
  D(\bar{\alpha}) = \frac{1}{n} \sum_{i=1}^n y_i + y_i \log \tfrac{\kappa_i \bar{\alpha}}{y_i}
              - \frac{\lambda}{2} \bigg\|
                  \frac{1}{\lambda n} \sum_{i=1}^n \kappa_i \bar{\alpha} x_i
                  - \frac{1}{\lambda} \psi
                \bigg\|^2.
\end{equation*}
This problem is concave in $\bar{\alpha}$ and the optimal $\bar{\alpha}$ is obtained by setting
the derivative to zero:
\begin{equation*}
D'(\bar{\alpha})
 = \frac{1}{n} \sum_{i=1}^n \frac{y_i}{\bar{\alpha}}
              - \bigg\langle
              \frac{1}{\lambda n} \bar{\alpha}  \sum_{i=1}^n \kappa_i x_i - \frac{1}{\lambda} \psi
              \; \Big| \;
              \frac{1}{n}  \sum_{i=1}^n \kappa_i x_i \bigg\rangle = 0.
\end{equation*}
This leads to the following second order equation
\begin{equation*}
\bigg\| \frac{1}{n} \sum_{i=1}^n \kappa_i x_i \bigg\|^2  \bar{\alpha}^2
- \bigg\langle\psi \; \Big| \; \frac{1}{n}\sum_{i=1}^n \kappa_i x_i \bigg\rangle \bar{\alpha}
- \frac{\lambda}{n}\sum_{i=1}^n y_i= 0,
\end{equation*}
which has a unique positive solution
\begin{equation*}
\bar{\alpha} = \frac{1}{2 \big\| \frac{1}{n} \sum_{i=1}^n \kappa_i x_i \big\|^2}
\left( \bigg\langle\psi \; \Big| \; \frac{1}{n}\sum_{i=1}^n \kappa_i x_i \bigg\rangle + \sqrt{
  \bigg\langle\psi \; \Big| \; \frac{1}{n}\sum_{i=1}^n \kappa_i x_i \bigg\rangle^2
+ 4 \frac{\lambda}{n} \sum_{i=1}^n y_i \bigg\| \frac{1}{n} \sum_{i=1}^n \kappa_i x_i \bigg\|^2
} \right),
\end{equation*}
and concludes the proof. $\hfill \qed$

% \bibliography{bibliothese}
% \bibliographystyle{abbrv}

% paragraph implementation_details (end)

% section speeding_up_sdca (end)

% section specifications_for_poisson_regression_and_hawkes_processes (end)

\end{document}